\documentclass[runningheads]{llncs}

\usepackage{eccv}

\usepackage{graphicx}
\usepackage{booktabs}

\usepackage{soul}

\usepackage{amsmath,amsthm}
\usepackage{multirow}
\usepackage{wrapfig}
\usepackage{graphicx}
\usepackage{xcolor}

\usepackage{booktabs}
\usepackage[table]{xcolor}
\definecolor{color3}{rgb}{0.95,0.95,0.95}
\usepackage{graphicx}
\usepackage{amsmath}
\usepackage{arydshln}
\usepackage{array}

\usepackage{comment}

\newtheorem{prop}{Proposition}

\usepackage[accsupp]{axessibility}

\usepackage[hidelinks]{hyperref}

\usepackage{orcidlink}

\begin{document}

\title{HRDiT: Training-Free High-Resolution Image Generation with Off-the-Shelf Diffusion Transformer Models} 

\titlerunning{HRDiT: Training-Free High-Resolution Image Generation}

\author{Yu Xue\inst{1}\orcidlink{0009-0006-3961-890X} \and
Haoxuan Qu\inst{1}\orcidlink{0000-0001-5054-3394}\thanks{Corresponding author.} \and
Zhuoling Li\inst{1}\orcidlink{0009-0004-2166-3353} \and
Hongbin Xu\inst{2}\orcidlink{0000-0002-3455-1527} \and
Jianxiong Yin\inst{3}\orcidlink{0000-0003-4686-2768} \and
Simon See\inst{3}\orcidlink{0000-0002-4958-9237} \and
Hossein Rahmani\inst{1}\orcidlink{0000-0003-1920-0371} \and
Jun Liu\inst{1}\orcidlink{0000-0002-4365-4165}}

\authorrunning{Y.~Xue et al.}

\institute{Lancaster University, Lancaster, UK \and
South China University of Technology, Guangzhou, China \and
NVIDIA AI Tech Centre, Singapore\\
\email{\{y.xue9,h.qu5,z.li81,h.rahmani,j.liu81\}@lancaster.ac.uk}, \email{hongbinxu1013@gmail.com}, \email{\{jianxiongy,ssee\}@nvidia.com}}

\maketitle

\begin{abstract}
  Training-free text-to-high-resolution image generation has recently attracted growing research attention. However, existing studies on this task primarily focus on adapting off-the-shelf U-Net-based diffusion models to high resolutions, with limited progress on adapting off-the-shelf Diffusion Transformer (DiT) models despite their strong text-to-image generation capabilities at limited resolutions. In this work, we find two key challenges particularly hindering the application of off-the-shelf DiT models for high-resolution image synthesis in a training-free manner, namely, spatial disorder and long generation time. To address these challenges, we propose a novel method tailored to adapt off-the-shelf DiT models for high-resolution image synthesis. Extensive experiments show the efficacy of our method. Our code is available at: \url{https://github.com/zylwithxy/HRDiT}.
\end{abstract}

\section{Introduction}
\label{sec:intro}

Text-to-high-resolution image generation aims to directly synthesize high-quality, large-resolution images from text prompts. 
It benefits diverse applications such as digital art creation \cite{train_higen2}, advertising \cite{hartmann2025power}, and game development \cite{game}. 
To tackle this task, existing methods can be broadly categorized into \textit{training-based} and \textit{training-free} approaches.
\textit{Training-based} methods \cite{hoogeboom2023simple,guo2024make,liu2025boosting,zhang2025ledit} optimize new generators specifically for high-resolution outputs. However, they generally require substantial computational resources and large-scale high-resolution datasets for effective training \cite{du2024demofusion,train_higen}, limiting their practicality.
Given these limitations, recently, \textit{training-free} methods have emerged as an appealing alternative.
Instead of optimizing new generators, these methods aim to unlock the potential of off-the-shelf text-to-image diffusion models typically trained at limited resolutions (e.g., $1024 \times 1024$), driving them to effectively generate higher-resolution images in a training-free manner.
Due to their simplicity and flexibility, such approaches have attracted growing attention recently \cite{he2023scalecrafter,qiu2024freescale,yang2025fam,zhang2024hidiffusion}.

Among existing training-free approaches, the early dominance of U-Net architectures in off-the-shelf text-to-image diffusion models has led most of them \cite{he2023scalecrafter,qiu2024freescale,yang2025fam,zhang2024hidiffusion} to focus on adapting U-Net-based diffusion models to high resolutions.
However, the landscape of off-the-shelf diffusion models has recently shifted: Diffusion Transformer (DiT)-based models now substantially outperform their U-Net counterparts at limited resolutions \cite{peebles2023scalable}, driving the rapid emergence of various powerful DiT-based off-the-shelf text-to-image models, such as Stable Diffusion 3 \cite{sd3} and FLUX \cite{blackforestlabs2024flux}. 
However, despite this shift, training-free approaches for adapting off-the-shelf DiT models to high-resolution generation remain largely underexplored.
Motivated by this research gap, in this work, we aim to narrow down this gap, \textit{exploring effective adaption of off-the-shelf DiT models for high-resolution image generation in a training-free manner}.

\begin{wrapfigure}[28]{r}{0.51\textwidth}
\centering
\includegraphics[width=\linewidth]{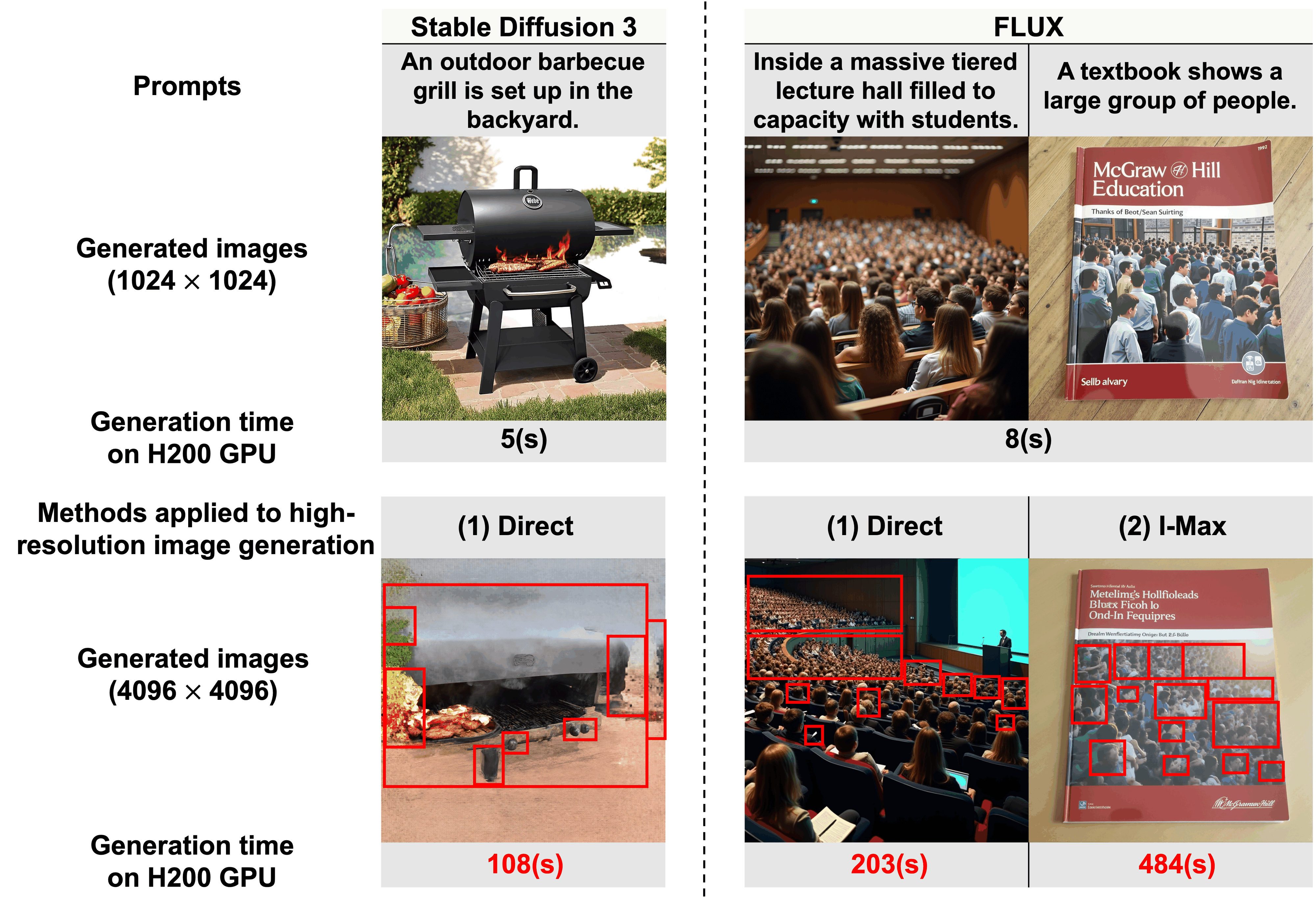}
\caption{Results of (1) \textit{directly} adapting off-the-shelf text-to-image DiT models to high-resolution ($4096$$\times$$4096$) generation, and (2) applying the DiT-tailored method I-Max~\cite{Imax}. At limited resolution ($1024$$\times$$1024$), off-the-shelf DiT models generate high-quality images quickly. However, when adapted to higher resolutions through \textit{direct} adaptation or I-Max, they exhibit spatial disorder (see red boxes and zoom-in) and incur long generation times (red text). Note that off-the-shelf DiT models can be \textit{directly} adapted since they mechanically accept higher-resolution initial noises, while I-Max is, by design, applicable to FLUX but not Stable Diffusion 3.}
\label{fig:intro_challenge}
\end{wrapfigure}
To achieve this, we first attempt to adapt off-the-shelf DiT models to higher resolutions either directly or by applying the existing DiT-tailored method I-Max~\cite{Imax}. Profiling across a wide range of generation tasks, we observe that, regardless of the adaptation strategy, off-the-shelf DiT models encounter two critical challenges that severely hinder their practicality at high resolutions.
We illustrate these challenges in Fig.~\ref{fig:intro_challenge} and summarize them as follows.
(1) \ul{Spatial disorder}: both adaptation strategies lead to the generation of images with noticeable spatial disorder (highlighted by red boxes in Fig.~\ref{fig:intro_challenge}), which markedly degrades overall image quality. (2) \ul{Long generation time}: when adapted for high-resolution image synthesis, off-the-shelf DiT models also suffer from very long generation time, further limiting their practicality. To effectively adapt off-the-shelf DiT models to high resolutions, we therefore aim to address these two challenges in the remainder of this work.

To this end, theoretical and empirical analyses are first drawn on to try to identify the key underlying causes of these challenges.
The theoretical analysis (detailed in Sec.~\ref{sec:3_1}) reveals that, a key reason behind the \textit{spatial disorder} issue can lie in the ``limited expressiveness''
of off-the-shelf DiT model’s positional embedding mechanism when adapted to high resolutions. 
Meanwhile, the empirical profiling shows that, multi-head attention computations are the dominant source of the \textit{long generation time} faced by DiT models at high resolutions. For example, when directly adapting mainstream DiT models (such as FLUX and Stable Diffusion~3) to generate 8K ($8192$$\times$$8192$) images, the multi-head attention computations alone account for over 90\% of the overall generation time for both models (as illustrated in Fig.~\ref{fig:attention_time}), taking 1,385 and 646 seconds per image, respectively.

\begin{wrapfigure}[19]{r}{0.6\textwidth}
  \centering
  \includegraphics[width=\linewidth]{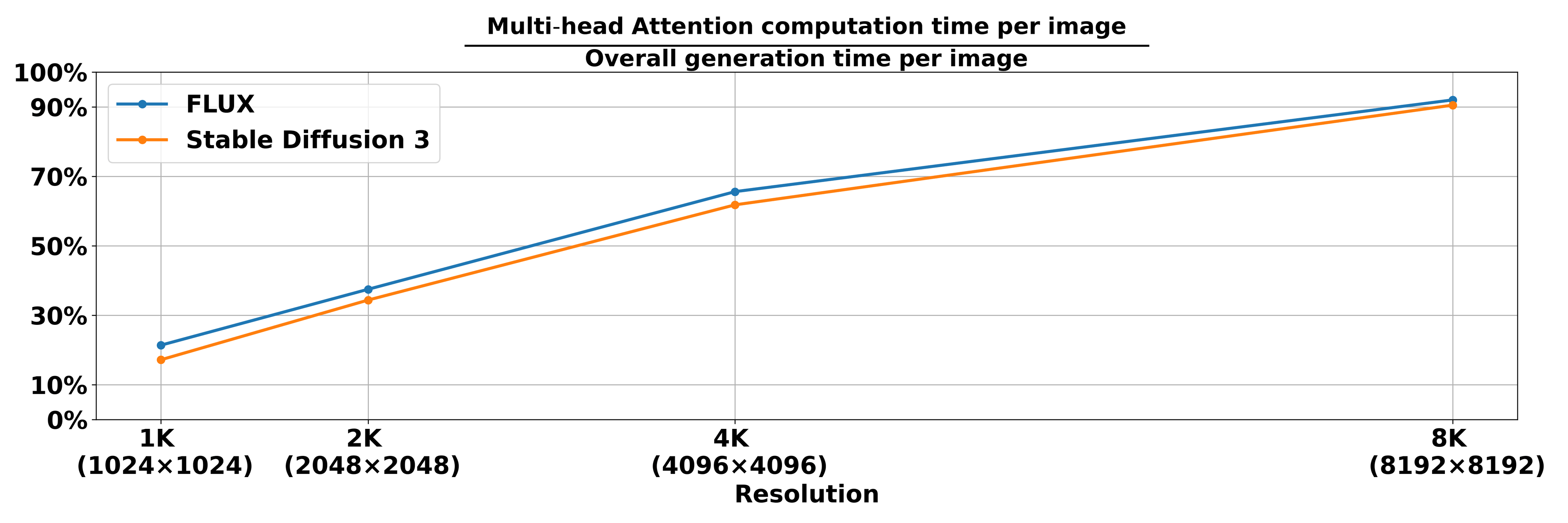}
   \caption{Proportion of time that mainstream off-the-shelf DiT models (FLUX~\cite{blackforestlabs2024flux} and Stable Diffusion 3~\cite{sd3}) spent on multi-head attention computations per image, relative to that on the overall generation process per image, across different image resolutions. These models are adapted to high resolutions in the same \textit{direct} manner as described in the caption of Fig.~\ref{fig:intro_challenge}. As shown, at high resolutions, multi-head attention computations alone dominate the overall generation time of these off-the-shelf DiT models.}
   \label{fig:attention_time}
\end{wrapfigure}
Building on these insights, to effectively adapt off-the-shelf DiT models to high resolutions, in this work, we propose a novel training-free framework, \textbf{H}igh-\textbf{R}esolution \textbf{DiT} (\textbf{HRDiT}). 
It comprises two components, each corresponding to one of the aforementioned challenges based on its identified underlying cause. 
To tackle the \textit{spatial disorder} challenge, inspired by \cite{li2025longdiff,han2023lm}, HRDiT involves a \textit{Spatial Position Alignment} (SPA) component, which adjusts the positional embedding mechanism of off-the-shelf DiT models via two complementary operations (Bundle and Slide) to help it preserve its proper functionality when extended to high resolutions. 
To tackle the \textit{long generation time} problem, inspired by \cite{fu2024moa}, HRDiT further integrates a \textit{Head-adaptive Attention Pruning} (HAP) component to effectively prune redundant computations in the multi-head attention, thereby substantially reducing runtime cost for high resolution generation. 
With these two components, HRDiT enables high-quality and efficient high-resolution image synthesis from off-the-shelf DiT models trained only at limited resolutions, entirely in a training-free manner.

\section{Related Work}
\label{sec:related}

\textbf{Training-free Text-to-high-resolution Image Generation}, which aims to unlock the high-resolution generation ability of off-the-shelf text-to-image diffusion models originally trained at limited resolutions, has recently gained significant attention
\cite{bar2023multidiffusion,Imax,kim2025diffusehigh,du2024demofusion,lin2024accdiffusion,zhang2024hidiffusion,huang2024fouriscale,he2023scalecrafter,qiu2024freescale,yang2025fam,hiflow,zhang2025frecas,wu2025megafusion,yang2025rectifiedhr,cao2024ap,qiu2025cinescale,lai2026pixelrush,vontobel2025hiwave,zhao2025ultraimage,koh2025scalediff,issachar2025dype}.
To tackle this task, some works, such as DemoFusion \cite{du2024demofusion}, DiffuseHigh \cite{kim2025diffusehigh}, and HiFlow \cite{hiflow}, focus primarily on modifying the diffusion process of diffusion models. Some others instead focus more on the architectures of diffusion models. For instance, ScaleCrafter \cite{he2023scalecrafter} employed 
dilation to expand U-Net’s receptive field. HiDiffusion \cite{zhang2024hidiffusion} dynamically adjusted feature map resolutions in the outer U-Net blocks to tackle the object repetition issue. FreeScale~\cite{qiu2024freescale} integrated restrained dilated convolutions and multi-scale fusion into the U-Net to enhance both global and local generation fidelity. 
Beyond U-Net-based designs, the architecture of DiT models has also been explored, e.g., by I-Max \cite{Imax}. Yet, we find that such investigations still leave two key challenges (i.e., spatial disorder and long generation time) that critically limit the effective application of off-the-shelf DiT models at high resolutions unaddressed.

In this work, we focus on the architectural perspective, and specifically target these two key challenges that significantly hinder DiT models from effective high-resolution generation. 
By identifying their key causes and correspondingly integrating proper strategies to address them, our framework enables off-the-shelf DiT models to achieve high-quality, efficient high-resolution synthesis without any retraining.

\section{The Proposed HRDiT Framework}
\label{sec:method}

Here, we aim to effectively adapt off-the-shelf DiT models 
to high-resolution image synthesis, by addressing two critical challenges that severely hinder their practical use at high resolutions: \textit{spatial disorder} and \textit{long generation time}. Drawing on analyses, we identify the key underlying causes of each challenge. Correspondingly, we propose \textbf{HRDiT}, a novel training-free framework integrated with two components, SPA and HAP,
respectively addressing aforementioned challenges based on identified underlying causes. Below, we detail these components.

\subsection{Spatial Position Alignment (SPA)}
\label{sec:3_1}

To handle the \textit{spatial disorder} challenge, we first examine its underlying cause, drawing inspiration from \cite{han2023lm}. 
Specifically, in DiT models, the positional embedding mechanism typically plays a crucial role in expressing positional information, distinguishing spatial locations within the generated image, and thus preventing spatial disorder \cite{lee2024groundit,bai2025positional,shen2025mmdit}.
This naturally raises a question: could the spatial disorder observed in off-the-shelf DiT models during high-resolution generation stem from the possible ``limited expressiveness'' of their positional embedding mechanisms, when extended to higher resolutions? To explore this, we first formalize the positional embedding mechanisms in off-the-shelf DiT models, which will then serve as the foundation for the subsequent analysis.

\noindent\textbf{Formalization of position embedding mechanism.} To facilitate a general analysis across different mainstream off-the-shelf DiT models, we aim to first formalize their position embedding mechanisms in a unified manner. In specific, we analyse that the positional embedding mechanisms of mainstream off-the-shelf DiT models, such as FLUX \cite{blackforestlabs2024flux} and Stable Diffusion 3 \cite{sd3}, though varying in detailed designs, can be unified under the following formulation: 
\begin{equation}
\setlength{\abovedisplayskip}{3pt}
\setlength{\belowdisplayskip}{3pt}
\begin{aligned}
\label{Eq:1}
\mathbf{q_i}=g_{q}(\mathbf{x_i},f^{tok}_{pe}(i)), \mathbf{k_i}=g_{k}(\mathbf{x_i},f^{tok}_{pe}(i)), \mathbf{v_i}=g_{v}(\mathbf{x_i},f^{tok}_{pe}(i)),
\end{aligned}
\end{equation}
where $\mathbf{x_i}$ denotes the $i$-th token in the input hidden state of an attention block, and $\mathbf{q_i}$, $\mathbf{k_i}$, $\mathbf{v_i}$ are the corresponding query, key, and value vectors derived from $\mathbf{x_i}$ via the functions $g_{q}(\cdot,\cdot)$, $g_{k}(\cdot,\cdot)$, and $g_{v}(\cdot,\cdot)$, respectively. 
To inject spatial awareness, the position embedding mechanisms feed the token-level positional signal $f^{tok}_{pe}(i)$ once into each of the three functions $g_{q}(\cdot,\cdot)$, $g_{k}(\cdot,\cdot)$, and $g_{v}(\cdot,\cdot)$. 
Additional details on this formulation are in Supplementary.

Denote $T$ the total number of visual tokens in the input hidden state. After obtaining the query, key, and value vectors for all $T$ tokens (i.e., $\{\mathbf{q_t}\}_{t=0}^{T-1}$, $\{\mathbf{k_t}\}_{t=0}^{T-1}$, and $\{\mathbf{v_t}\}_{t=0}^{T-1}$) via Eq.~\ref{Eq:1}, the attention in off-the-shelf DiT models can then be performed in a spatial-aware manner as:
\begin{equation}
\setlength{\abovedisplayskip}{3pt}
\setlength{\belowdisplayskip}{3pt}
\label{Eq:2}
\mathbf{c_{i, j}} = \frac{e^{\mathbf{q_i}\mathbf{k_j}^\intercal}}{\sum_{t=1}^{T} e^{\mathbf{q_i}\mathbf{k_t}^\intercal}} \mathbf{v_j},
\end{equation}
where $\mathbf{c_{i,j}}$ denotes the contribution of token $j$ to the attention output at position $i$, and $e$ is the Euler number. For simplicity, we omit the stabilization factor $\sqrt{d}$, where $d$ is the dimensionality of each key vector, in Eq.~\ref{Eq:2}. At this point, we arrive at a unified formalization of how positional embedding mechanisms in mainstream off-the-shelf DiT models are designed to inject spatial awareness into their attention operations, which underpins the analysis below.

\noindent\textbf{Expressiveness analysis.} Based on our above unified formalization of positional embedding mechanisms in mainstream off-the-shelf text-to-image DiT models, we next ask whether existing theoretical tools can be drawn on to analyze the expressiveness of these mechanisms when they are extended to high resolutions with substantially more tokens, i.e., much larger $T$. To this end, we observe that the pseudo-dimension-based analysis developed in prior work~\cite{han2023lm} for relative positional embeddings can, when combined with our formalization, be adapted to perform the desired analysis under mild assumption, across mainstream off-the-shelf DiT models including FLUX and Stable Diffusion 3, \textit{regardless of whether the positional embeddings are relative or not}. As elaborated below, this adapted theoretical analysis reveals that \ul{these mechanisms can indeed suffer from limited expressiveness when extended to high resolutions}.

To set up the adapted analysis, we introduce the following notations. (1) By combining Eq.~\ref{Eq:1} and Eq.~\ref{Eq:2}, we first express $\mathbf{c_{i,j}}$ via a single function $g$ as:
\begin{equation}
\setlength{\abovedisplayskip}{3pt}
\setlength{\belowdisplayskip}{3pt}
\label{Eq:2_5}
\mathbf{c_{i,j}} = g(\mathbf{x_i}, \mathbf{x_j}, f_{pe}(i, j)),
\end{equation}
where $g$ represents the complete computation performed by off-the-shelf DiT models to obtain $\mathbf{c_{i,j}}$ from the input hidden-state tokens $\mathbf{x_i}$ and $\mathbf{x_j}$. Expressing $\mathbf{c_{i,j}}$ in this way, the positional information involved in computing $\mathbf{c_{i,j}}$ can be compactly represented by a single \textit{pairwise positional signal} denoted as $f_{pe}(i, j)$, rather than by two separate token-level positional signals $f^{tok}_{pe}(i)$ and $f^{tok}_{pe}(j)$ (see supplementary for more details). This compact representation facilitates the subsequent analysis. (2) We further denote $S_{pe} \;=\; \{\, f_{pe}(i,j) \mid i \in \{0,\dots,T-1\},~ j \in \{0,\dots,T-1\} \,\}$ as the set of all possible pairwise positional signals. With these notations established, we now proceed to the adapted theoretical analysis.

\begin{prop}
\label{theorem:1}
Let $dis_{g}$ denote a distance measure quantifying the discrepancy between two pairwise positional signals in the output space of the function $g$ (see Supplementary for exact definition of $dis_{g}$).
By definition of $g$ in Eq.~\ref{Eq:2_5}, $dis_{g}$ then measures the discrepancy between two pairwise positional signals, after they are respectively propagated through $g$ during attention computation.
Using $dis_{g}$, we partition all elements in the set $S_{pe}$ into $h(T)$ mutually exclusive subsets such that any two elements within the same subset satisfy $dis_{g} \le \epsilon$, while any two elements from different subsets satisfy $dis_{g} > \epsilon$, where $\epsilon$ is a predefined threshold.
Intuitively, the resulting $h(T)$ characterizes the positional embedding mechanism’s ability to distinguish between different pairwise positional signals in $S_{pe}$ after attention computation.
A larger $h(T)$, indicating that $S_{pe}$ can be divided into a greater number of well-separated subsets under $dis_{g}$, corresponds to a stronger distinction capability of the positional embedding mechanism.
Yet meanwhile, $h(T)$ can be proved to be upper-bounded by:
\begin{equation}
\setlength{\abovedisplayskip}{3pt}
\setlength{\belowdisplayskip}{3pt}
\label{eq:3}
h(T) \;\le\; \lambda \cdot \Big(\sup_{S_{pe}} |g(\mathbf{x_i}, \mathbf{x_j}, f_{pe}(i,j))|\Big)^{2\xi},
\end{equation}
where $\xi$ denotes the pseudo-dimension of the function class $\mathcal{G} = \{\, g(\cdot, \cdot, f_{pe}(i,j)) \mid i \in \{0, \dots, T-1\},~ j \in \{0, \dots, T-1\} \,\}$, and $\lambda = (\frac{e}{\epsilon})^{2\xi} \cdot 2^{4\xi+1}$. 
Notably, $\lambda$ can be regarded as a constant when the threshold $\epsilon$ is fixed.
\end{prop}
The proof of Proposition~\ref{theorem:1} is provided in Supplementary. Note that, ideally, for the positional embedding mechanism to function properly, $h(T)$, which measures the mechanism's ability to express different pairwise positional signals in a distinguishable manner after attention computation, should reach its maximum possible value, namely the size $|S_{pe}|$ of the set $S_{pe}$ (i.e., satisfy $h(T)=|S_{pe}|$). Otherwise, at least two pairwise positional signals in the set $S_{pe}$ would have cross-distances below $\epsilon$ after attention computation, making them indistinguishable and thereby inducing spatial disorder. 
However, Proposition~\ref{theorem:1} suggests that $h(T)=|S_{pe}|$ may not hold at high resolutions. This is because, (1) as resolution increases, the number of tokens $T$ grows rapidly, and $|S_{pe}|$ typically scales accordingly (see Supplementary for details).
(2) Yet meanwhile, Eq.~\ref{eq:3} shows that the attainable $h(T)$ is upper-bounded by the supremum of $g(\cdot)$ (up to a constant factor $\lambda$ and a constant power $2\xi$). Taken together, these two factors indicates that, as resolution increases, achieving $h(T)=|S_{pe}|$ would necessitate the aforementioned upper bound of $h(T)$ to scale at least as fast as $|S_{pe}|$. 
The associated analysis, however, suggests that this scaling may fail to occur in practice. Taking FLUX as an example, when an analysis similar to that in \cite{li2025longdiff} is applied to 4K image generation, fewer than 10\% of the test samples are found to satisfy Eq.~\ref{eq:3} after substituting $h(T)$ with $|S_{pe}|$, and this ratio can further decrease at higher resolutions.
The above indicate that, \ul{when adapted for high-resolution generation, the positional embedding mechanisms in off-the-shelf text-to-image DiT models can exhibit limited expressiveness, failing to adequately distinguish pairwise positional signals, thus leading to observed spatial disorder}.

\noindent\textbf{Addressing limited expressiveness through SPA.} Building on the above analysis, we aim to address the limited expressiveness of DiT models’ positional embedding mechanisms to handle the spatial disorder problem. Achieving this addressing, however, is non-trivial, particularly because our goal is to do so in a training-free and universally compatible manner across diverse mainstream off-the-shelf DiT models. To tackle this challenge, we note that, though different mainstream DiT models may differ in their specific implementations of the positional function $f_{pe}$, they all share a common structure: $f_{pe}$ takes token indices (i.e., $i$ and $j$) as inputs. This observation suggests that if we can achieve the addressing by directly manipulating these token indices $i$ and $j$ before they are fed into $f_{pe}$, the addressing process can be carried out in a universally compatible manner. 
To this end, we take inspiration from \cite{li2025longdiff,jin2024llm}, adopt a similar approach, and integrate a training-free component, termed Spatial Position Alignment (SPA), into our framework to achieve the addressing in such a way.

Specifically, SPA consists of two key operations: \textbf{bundle} and \textbf{slide}. In the \textbf{bundle} operation, to facilitate $h(T)$ in attaining $|S_{pe}|$ (i.e., to satisfy $h(T) = |S_{pe}|$) and thus help the positional embedding mechanism to function properly even at high resolutions, we aim to reduce $|S_{pe}|$. 
To achieve this by manipulating token indices before they are fed into the positional function $f_{pe}$, we group tokens into several bundles and feed the corresponding bundle indices, rather than individual token indices, into $f_{pe}$. In this way, as the number of grouped bundles is typically much smaller than the number of tokens, we can constrain the input diversity of $f_{pe}$ and thus reduce $|S_{pe}|$, which reflects the output diversity of $f_{pe}$ based on the definition of $S_{pe}$. In turn, as elaborated in the above analysis associated to Proposition~\ref{theorem:1}, this helps the positional embedding mechanism better preserve the distinguishability among positional signals in $S_{pe}$ and restore its proper functionality. However, despite such benefit, bundling also introduces an inherent limitation: positional distinctions among tokens within the same bundle become indistinguishable, potentially weakening fine-grained spatial discrimination in generated images. To tackle this, SPA integrates a complementary \textbf{slide} operation, which further equips the framework with fine-grained positional distinguishability among tokens within each bundle. Together, these two operations effectively rectify the positional embedding mechanism, facilitating off-the-shelf DiT models to preserve spatial coherence even when adapted to high resolutions.
Below, we detail these two operations.

\ul{(1) The bundle operation.} In this operation, we start from the original positional function $f_{pe}(i, j)$, whose inputs are the token indices themselves (i.e., $i$ and $j$).
Our goal is to group tokens into bundles and feed $f_{pe}$ with the corresponding bundle indices instead.
Formally, this rewrites $f_{pe}(i, j)$ as $f_{pe}(\phi_{bundle}(i), \phi_{bundle}(j))$, where $\phi_{bundle}(i)$ and $\phi_{bundle}(j)$ denote the indices of the bundles containing the $i$-th and $j$-th tokens, respectively.
For simplicity, we here bundle tokens sequentially according to their token indices, so that the resulting bundle indices approximately preserve the original positional order among the $T$ tokens. 
In other words, tokens belonging to bundles with larger bundle indices naturally possess larger token indices than those tokens in bundles with smaller bundle indices.
Under this setup, $\phi_{bundle}(\cdot)$ is defined as follows:
\begin{equation}
\setlength{\abovedisplayskip}{3pt}
\setlength{\belowdisplayskip}{3pt}
\phi_{bundle}(i) =
\begin{cases}
0, & 0 \le i < N_1, \\[3pt]
\displaystyle \left\lceil \dfrac{i + 1 - N_1}{N} \right\rceil, & N_1 \leq i < T,
\end{cases}
\end{equation}
where $N$ is a hyperparameter, and $N_1$ ranges from $1$ to $N$. $N_1$ and $N$ denote the number of tokens in the first bundle and in each middle bundle (excluding the first and last ones), respectively. Here, the first bundle is intentionally not enforced to reach the same size as the middle ones, which facilitates the subsequent \textbf{slide} operation. 
Furthermore, given $N_1$ and $N$, the size of the last bundle can be automatically determined as $T - N_1 - \left\lfloor \frac{T - N_1}{N} \right\rfloor \!\times N$, so it does not need to be explicitly specified.

\begin{wrapfigure}[23]{r}{0.5\textwidth}
\centering
\includegraphics[width=\linewidth]{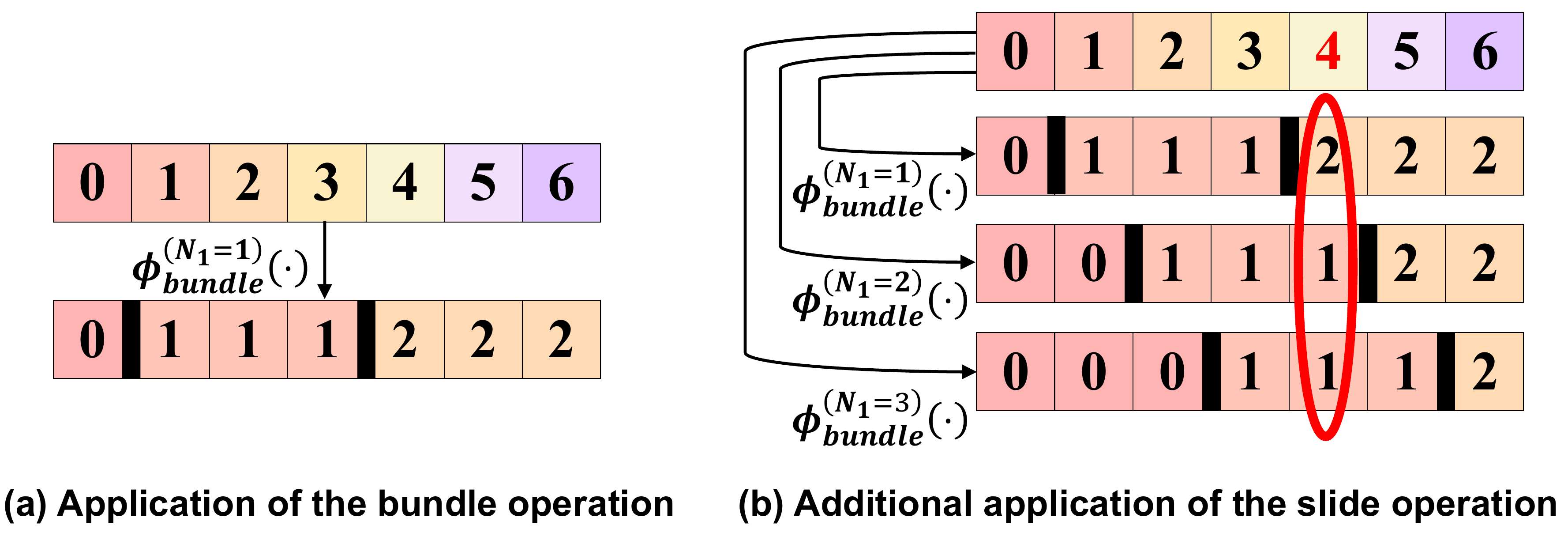}
\caption{Illustration of the \textbf{bundle} and \textbf{slide} operations, with $T = 7$ and $N = 3$. (a) The \textbf{bundle} operation with $N_1 = 1$: using the mapping function $\phi_{bundle}(\cdot)$ with $N_1 = 1$ (i.e., $\phi_{bundle}^{(N_1 = 1)}(\cdot)$), the $T = 7$ token indices from 0 to 6 are mapped to $\left\lceil \frac{T + N - N_1}{N} \right\rceil = 3$ bundle indices (0 to 2), reducing the input diversity of $f_{pe}$ when its inputs are bundle indices rather than token indices. (b) Additionally applying the \textbf{slide} operation introduces $N = 3$ variants of the mapping function, each with a distinct $N_1$ value. Progressively varying $N_1$ (from 1 to 3) produces bundles with sliding boundaries across successive bundling procedures (the thick black line segments sliding rightward).}
   \label{fig:slide_example}
\end{wrapfigure}
An illustrative example of the application of the $\phi_{bundle}(\cdot)$ function is provided in Fig.~\ref{fig:slide_example}(a). As shown, applying $\phi_{bundle}(\cdot)$ reduces the diversity of each input to $f_{pe}(\cdot, \cdot)$, from the total number of token indices $T$ to the total number of bundle indices, which can be computed as $\left\lceil \frac{T + N - N_1}{N} \right\rceil$. 
Accordingly, when computing $|S_{pe}|$ based on $f_{pe}$ with bundle indices as its inputs, we can replace $T$ with $\left\lceil \frac{T + N - N_1}{N} \right\rceil$, yielding a significantly smaller $|S_{pe}|$. 
For instance, in the FLUX model, instead of $|S_{pe}| = 2T - 1$, we now have $|S_{pe}| = 2 \times \left(\left\lceil \frac{T +N - N_1}{N} \right\rceil\right) - 1$ after applying the bundle operation.
Following the analysis associated with Proposition~\ref{theorem:1}, this reduction enhances the distinguishability among elements in $S_{pe}$, and thus helps restore the proper functionality of the positional embedding mechanism.

\ul{(2) The slide operation.} However, applying the bundle operation alone cannot fully resolve the spatial disorder problem, as it eliminates positional distinctions among tokens within the same bundle, thereby weakening fine-grained spatial discrimination. To tackle this issue, a \textbf{slide} operation is further integrated to recover intra-bundle positional distinctions, with the key idea of using \textit{sliding bundle boundaries}, as elaborated below.

Specifically, given the bundle size $N$ for each middle bundle, the slide operation constructs $N$ variants of the mapping function $\phi_{bundle}(\cdot)$, denoted as $\{\phi_{bundle}^{(N_1 = 1)}(\cdot), \dots, \phi_{bundle}^{(N_1 = N)}(\cdot)\}$, where each variant adopts a distinct first-bundle size $N_1$. As illustrated in Fig.~\ref{fig:slide_example}(b), constructed in this way, these variants produce bundles whose boundaries progressively slide as $N_1$ varies from $1$ to $N$. Importantly, with such sliding, by collectively considering all $N$ mapping functions rather than any single one in isolation, the desired intra-bundle position distinctions can be achieved among tokens.
For example, as shown in the red ellipse of Fig.~\ref{fig:slide_example}(b), the token with index $4$ is distinctively (uniquely) represented by the set of $N = 3$ bundle indices $\{2, 1, 1\}$ derived respectively from the $N = 3$ mapping functions, and no any other token corresponds to this same set. 
We also prove in Supplementary that this representation approach always uniquely represents every token index.
The above means that, the collective use of all mapping functions constructed in the slide operation ensures clear positional distinctions among all tokens. Then, as the final step of the slide operation, to realize this collective use in a training-free manner, we compute:
\begin{equation}
\setlength{\abovedisplayskip}{3pt}
\setlength{\belowdisplayskip}{3pt}
\label{eq:slide}
\mathbf{c^{SPA}_{i, j}} = \frac{1}{N} \sum_{n=1}^{N} g\bigg(\mathbf{x_i}, \mathbf{x_j}, f_{pe}\Big(\phi_{bundle}^{(N_1 = n)}(i), \phi_{bundle}^{(N_1 = n)}(j)\Big)\bigg),
\end{equation}
where $\mathbf{c^{SPA}_{i, j}}$ (replacing $\mathbf{c_{i, j}}$ in Eq.~\ref{Eq:2_5}) denotes the contribution of token $j$ to the attention output at position $i$ after incorporating SPA into our framework. Intuitively, Eq.~\ref{eq:slide} applies $g(\cdot)$ multiple times in parallel, once under each mapping function, and then averages their corresponding results. In the $n$-th application of $g(\cdot)$, the only modification introduced by SPA relative to the off-the-shelf DiT model is that the token indices $i$ and $j$ are replaced by their corresponding bundle indices $\phi_{\text{bundle}}^{(N_1 = n)}(i)$ and $\phi_{\text{bundle}}^{(N_1 = n)}(j)$ before being passed to the positional function $f_{pe}$, while all other parts remain unchanged. The mapping from token indices to bundle indices is illustrated in Fig.~\ref{fig:slide_example}.

In summary, to tackle the \textit{spatial disorder} problem, SPA simply needs to replace $\mathbf{c_{i, j}}$ with $\mathbf{c^{SPA}_{i, j}}$ in the attention computation of off-the-shelf DiT models at high resolutions. It is worth noting that, since the $N$ token-to-bundle index mapping procedures in Eq.~\ref{eq:slide} can be executed in parallel, SPA generally incurs only a minor increase in overall generation time (e.g., from 203 to 212 seconds on average per 4K ($4096 \times 4096$) image generation, as reported in Tab.~\ref{Tab:ablation_1}).

\subsection{Head-adaptive Attention Pruning (HAP)}
\label{sec:3_2}

Above, we handle the \textit{spatial disorder} challenge. Yet, as illustrated in Fig.~\ref{fig:intro_challenge}, the \textit{long generation time} problem still limits the practicality of adapting off-the-shelf DiT models for high-resolution synthesis. To address this problem, the underlying cause is first profiled. As shown in Fig.~\ref{fig:attention_time}, the profiling reveals that, \ul{when off-the-shelf DiT models are directly applied to higher testing resolutions, \textit{multi-head attention} computations rapidly dominate the overall generation time due to the quadratic complexity of attention}. This profiling motivates a natural strategy for accelerating DiT’s high-resolution generation process: focus on multi-head attention as the \textit{dominant cause} of the \textit{long generation time} problem and prune ``redundant’’ attention computations that contribute little to final generation quality.

To realize this idea, we first examine whether existing vision transformer pruning techniques in image can be well-suited to perform ``redundant'' attention pruning in our training-free text-to-high-resolution image generation problem. We however find them ill-suited. Many existing pruning methods \cite{zheng2022savit,tang2022patch} rely on additional training, which directly contradicts the training-free constraint of our problem. To circumvent this issue, recent works \cite{yuan2024ditfastattn,liu2024clear} introduce local-window-based attention, which prunes computations by focusing attention to local regions (windows), enabling reducing runtime in a training-free manner. Yet, these methods remain ill-suited when facing our problem, i.e., when handling high resolutions, in which long-range interactions are essential for maintaining generation fidelity~\cite{leroy2024winwin}, local-window-based approaches~\cite{liu2024clear} typically still require additional training to prevent noticeable quality degradation, again conflicting with our training-free constraint. 

Given these incompatibilities, addressing the \textit{long generation time} of off-the-shelf DiT models at high resolutions calls for inspiration beyond conventional vision pruning. In particular, motivated by recent progress in attention pruning for long-context text generation \cite{fu2024moa}, we adapt this idea to the training-free text-to-high-resolution image generation setting. Specifically, we integrate an attention pruning component into HRDiT, termed \textit{Head-adaptive Attention Pruning} (HAP), which enables effective attention pruning in our context without additional training. Overall, HAP builds on the local-window-based paradigm to take advantage of its potential training-free applicability. Meanwhile, HAP seeks to further address the paradigm's tendency to cause noticeable generation quality degradation at high resolutions when equipped with no additional training.

For HAP to achieve the latter, we draw on a key observation from prior works~\cite{zhang2025fast,wu2025retrieval,fu2024moa}: once trained, different heads in a Transformer’s multi-head attention tend to assume distinct yet relatively stable roles, each focusing on a particular \textit{attention scope}. Specifically, for a given head, when computing the attention output at position $i$, only tokens within a certain scope around that position contribute meaningfully, and this scope may vary across heads rather than being uniformly local. For example, some heads specialize in long-range interactions and operate over broad scopes, whereas others focus on local structures within narrow neighborhoods. Building on this inspiration, we further validate that a similar head-dependent attention-scope pattern also emerges in our high-resolution image generation setting.

Motivated by the above, a natural way to prune redundant attention computations when scaling off-the-shelf DiT models to high-resolution generation is hypothesized to be the adaptive allocation of attention windows across heads according to each head’s inherent attention scope, rather than the uniform application of identical local windows to all heads as in conventional local-window-based strategies. Ideally, such head-adaptive window allocation allows more accurate identification of truly redundant computations, thereby enabling attention pruning that preserves high generation quality while substantially improving efficiency. However, determining appropriate attention scopes for each head in practice remains non-trivial.
To address this, inspired by \cite{fu2024moa}, a two-step preparatory procedure is adopted in HAP to derive head-specific scopes from the off-the-shelf DiT model before inference. 
Once derived, these head-specific scopes can then be applied as per-head attention windows during inference, enabling substantial runtime reduction while preserving generation fidelity. The two preparatory steps adopted in HAP are detailed below.

\ul{Preparatory Step 1: appropriateness quantification for each candidate scope.} To determine an appropriate attention scope for each head, the effect of each candidate scope on computational cost and generation quality is first quantified.
Formally, let $N_{\text{head}}$ denote the total number of attention heads in an off-the-shelf DiT model, and assign each head $N_{\text{scope}}$ candidate scopes, where $N_{\text{scope}}$ is a hyperparameter.
For each candidate scope $n_{\text{scope}} \in \{1, \dots, N_{\text{scope}}\}$ of each head $n_{\text{head}} \in \{1, \dots, N_{\text{head}}\}$, two indicators are defined:
$I_c(n_{\text{head}}, n_{\text{scope}})$, representing the computational cost of this head under the given scope, and $I_q(n_{\text{head}}, n_{\text{scope}})$, representing the degradation in generation quality when this head operates under the given scope instead of the full attention scope covering all $T$ tokens.
For simplicity, given a head with $T$ input tokens, its $N_{\text{scope}}$ candidate scopes are uniformly defined to span all possible coverage ratios of these tokens.
Specifically, under the $n_{\text{scope}}$-th candidate, the attention output at each position is computed by summing only the contributions from the $\frac{n_{\text{scope}}}{N_{\text{scope}}} \times T$ tokens neighboring that position, with all $T$ tokens being covered only when $n_{\text{scope}} = N_{\text{scope}}$.
Based on this setup, we next describe how $I_c(n_{\text{head}}, n_{\text{scope}})$ and $I_q(n_{\text{head}}, n_{\text{scope}})$ are computed.

For $I_c(n_{\text{head}}, n_{\text{scope}})$, note that the computational complexity of producing the attention output at each position scales linearly with the number of tokens participating in that position’s attention computation. 
Accordingly, the total attention computational cost for the $n_{\text{head}}$-th head under the $n_{\text{scope}}$-th scope, i.e., $I_c(n_{\text{head}}, n_{\text{scope}})$, can be simply quantified by the total number of participating tokens involved in its attention computations, aggregated over all positions.
Deriving $I_q(n_{\text{head}}, n_{\text{scope}})$, however, is more difficult. 
Naively, a brute-force estimation of $I_q(n_{\text{head}}, n_{\text{scope}})$ across all $N_{\text{head}}$ heads and $N_{\text{scope}}$ candidate scopes could require at least $N_{\text{head}} \times (N_{\text{scope}} - 1)$ DiT model passes, each followed by a loss computation on the generated image to evaluate its impact on generation quality (see Supplementary for details).
This can make the preparatory stage by itself already computationally prohibitive. To accelerate the above brute-force estimation procedure, a natural attempt shall be to use losses computed on intermediate features as efficient proxies for losses computed on the final generated images \cite{zhang2025ditfastattnv2}. However, we find that, at high resolutions, such intermediate signals can correlate poorly with final image quality, making them unreliable estimators of $I_q(n_{\text{head}}, n_{\text{scope}})$.
To overcome the limitations of the above strategies, inspired by \cite{fu2024moa}, HAP adopts its efficient estimation strategy, where careful mathematical derivations show that \ul{a single model pass under the full attention scope suffices to \textit{efficiently yet reliably} estimate $I_q(n_{\text{head}}, n_{\text{scope}})$ across all heads and candidate scopes}.
Formally, under this estimation approach, by leveraging a Taylor expansion and accounting for the softmax operation in DiT attention, $I_q(n_{\text{head}}, n_{\text{scope}})$ can be estimated using only a single model pass, as
\begin{equation}
\begin{aligned}
\label{eq_Iq}
\setlength{\abovedisplayskip}{3pt}
\setlength{\belowdisplayskip}{3pt}
I_q(n_{\text{head}}, n_{\text{scope}}) \approx
&  
\sum_{(u, v) \in S^{(n_{\text{scope}})}_{\text{omit}}} \Big(\frac{\partial L}{\partial A(u,v)} \cdot \big(-A(u,v)\big) \quad +
\\ &  
\sum_{w\in\{1, \dots, T\} \setminus \{v\}
} \frac{\partial L}{\partial A(u,w)} \cdot \frac{ A(u,w)  \cdot A(u,v)}{1 - A(u,v)}\Big),
\end{aligned}
\end{equation}
where $L$ is the loss computed on the image generated from this single model pass, $A \in \mathbb{R}^{T\times T}$ is the attention score matrix in this model pass corresponding to the $n_{\text{head}}$-th head, and $A(u, v)$ denotes its $(u,v)$-th element.
Here, $S^{(n_{\text{scope}})}_{\text{omit}}$ denotes the set of index pairs $(u,v)$ that fall outside the $n_{\text{scope}}$-th attention scope, i.e., pairs for which when applying the $n_{\text{scope}}$-th attention scope and computing the attention output
at position $u$, the contribution of token $v$ (i.e., $\mathbf{c^{SPA}_{u,v}}$) is omitted. 
In practice, to stabilize the estimation, $I_q$ is further averaged over different input prompts rather than estimated from a single input prompt. Despite this, computing $I_q(n_{\text{head}}, n_{\text{scope}})$ remains lightweight, taking only less than an hour in total across all heads and candidate scopes during the one-time preparatory stage (see Tab.~\ref{Tab:t_pre} in Supplementary).

\ul{Preparatory Step 2: appropriate scope determination.}
At this point, quantitative estimates have been obtained for how different candidate scopes, when assigned to each head in DiT, affect both computational cost and generation quality. The work left is then to assign an appropriate scope to each head such that the collective scope assignment achieves the best trade-off between efficiency and generation quality.

Formally, let $r_c \in [0,1]$, as a hyperparameter, denote the target ratio of total computational cost for all multi-head attention operations, measured relative to that under full attention (i.e., $r_c = 1$ corresponds to using full attention). The problem of determining the optimal collective scope assignment can then be formulated as the optimization problem of minimizing the total generation quality degradation across all heads under this computational constraint:
\begin{equation}
\setlength{\abovedisplayskip}{2pt}
\setlength{\belowdisplayskip}{0pt}
\begin{aligned}
& \mathop{\arg\min}_{(n_{\text{scope}}^{(1)}, \dots, n_{\text{scope}}^{(N_{\text{head}})})}~~~
\;\sum_{n_{\text{head}}=1}^{N_{\text{head}}} I_q(n_{\text{head}}, n_{\text{scope}}^{(n_{\text{head}})}) \\
\text{s.t.}\quad 
&\sum_{n_{\text{head}}=1}^{N_{\text{head}}} I_c(n_{\text{head}}, n_{\text{scope}}^{(n_{\text{head}})}) \le r_c \times \sum_{n_{\text{head}}=1}^{N_{\text{head}}} I_c(n_{\text{head}}, N_{\text{scope}}).
\end{aligned}
\label{eq:7}
\end{equation}
where $n_{\text{scope}}^{(n_{\text{head}})}$ denotes the $n_{scope}$-th candidate scope of the $n_{head}$-th head. Though this optimization problem appears complex, following MoA~\cite{fu2024moa}, it can be reformulated as an integer programming problem, enabling Eq.~\ref{eq:7} to be efficiently solved using linear solvers (e.g., Gurobi~\cite{gurobi}) within minutes.

Through the above two preparatory steps, we obtain a collective scope assignment that can substantially enhance generation efficiency while maintaining high generation quality, requiring only a few DiT passes and a single call to an existing linear solver. Once derived, this assignment can be precompiled into GPU kernels (e.g., via FlashInfer~\cite{ye2025flashinfer}).
During inference, the precompiled assignment can then be applied to eliminate redundant attention computations in a cost-free manner, thereby effectively addressing the \textit{long generation time} issue encountered by off-the-shelf DiT models when adapted to high resolutions.

\subsection{Overall Testing}
\label{sec:3_3}

For testing at the high resolution, before testing, HRDiT first performs the one-time preparatory stage described in Sec.~\ref{sec:3_2}, to determine an appropriate collective scope assignment plan across different attention heads. For per-input generation (testing), HRDiT adopts the determined scope plan to enhance efficiency (handling the \textit{long generation time} problem). Notably, in both the one-time preparatory stage and the per-input generation process, HRDiT computes the per-token contribution $\mathbf{c^{SPA}_{i,j}}$ via Eq.~\ref{eq:slide} to address the \textit{spatial disorder} issue.

\section{Experiments}

\subsection{Experimental Settings \& Implementation Details.} 

\begin{table}[t]
  \caption{Results on 2K and 4K image generation. Notably, by their design, I-Max is applicable to FLUX but not to Stable Diffusion 3 (SD3), whereas FreCaS is applicable to SD3 but not to FLUX.}
  \centering
  \resizebox{\linewidth}{!}{
  \begin{tabular}{c|clcccccc}
    \hline
    Resolution & Model Type &
    Method & FID $\downarrow$ & FID$_\text{p}\downarrow$ & KID $\downarrow$ & KID$_\text{p}\downarrow$ & CLIP Score $\uparrow$ & Latency(s) $\downarrow$ \\
    \hline
    \multirow{13}{*}{$\mathbf{2k}$} & U-Net backbone & SDXL~\cite{sdxl} + FreeScale~\cite{qiu2024freescale} &73.23 &65.04 &0.0093 &0.0132 &30.57 &26 \\
    \cline{2-9}
     & \multirow{12}{*}{DiT-based backbone} & SD3~\cite{sd3} (\textit{Direct}) &81.34 &70.88 &0.0215 &0.0194 &30.23 &17 \\
    & & SD3~\cite{sd3} + DemoFusion~\cite{du2024demofusion} &72.21 &59.24 &0.0132 &0.0125 &31.40 &35 \\
    & & SD3~\cite{sd3} + DiffuseHigh~\cite{kim2025diffusehigh} &70.57 &57.39 &0.0119 &0.0112 &31.59 &20 \\
    & & SD3~\cite{sd3} + FreCaS~\cite{zhang2025frecas} &73.27 &63.55 &0.0139 &0.0132 &31.41 &16 \\
    & & SD3~\cite{sd3} + HiFlow~\cite{hiflow} &70.49 &57.35 &0.0126 &0.0110 &31.62 &18 \\
    & & SD3~\cite{sd3} + Ours &\textbf{67.90} &\textbf{52.59} &\textbf{0.0081} &\textbf{0.0073} &\textbf{31.81} &\textbf{11} \\
    \cline{3-9}
    & & FLUX~\cite{blackforestlabs2024flux} (\textit{Direct}) &73.96 &69.56 &0.0173 &0.0164 &29.63 &47 \\
    & & FLUX~\cite{blackforestlabs2024flux} + DemoFusion~\cite{du2024demofusion} &68.95 &60.53 &0.0127 &0.0120 &30.83 &94 \\
    & & FLUX~\cite{blackforestlabs2024flux} + DiffuseHigh~\cite{kim2025diffusehigh} &67.40 &58.63 &0.0115 &0.0108 &31.01 &52 \\
    & & FLUX~\cite{blackforestlabs2024flux} + I-Max~\cite{Imax} &68.37 &56.96 &0.0121 &0.0106 &30.67 &85 \\
    & & FLUX~\cite{blackforestlabs2024flux} + HiFlow~\cite{hiflow} &67.45 &56.79 &0.0117 &0.0105 &30.73 &49 \\
    & & FLUX~\cite{blackforestlabs2024flux} + Ours &\textbf{64.42} &\textbf{51.55} &\textbf{0.0074} &\textbf{0.0067} &\textbf{31.27} &\textbf{35} \\
    \hline
    \multirow{13}{*}{$\mathbf{4k}$} & U-Net backbone & SDXL~\cite{sdxl} + FreeScale~\cite{qiu2024freescale} &74.85 &68.89 &0.0112 &0.0157 &30.49 &248 \\
    \cline{2-9}
     & \multirow{12}{*}{DiT-based backbone} & SD3~\cite{sd3} (\textit{Direct}) &84.29 &74.45 &0.0243 &0.0231 &30.18 &108 \\
    & & SD3~\cite{sd3} + DemoFusion~\cite{du2024demofusion} &73.39 &60.75 &0.0141 &0.0133 &31.16 &229 \\
    & & SD3~\cite{sd3} + DiffuseHigh~\cite{kim2025diffusehigh} &71.88 &58.76 &0.0130 &0.0126 &31.35 &115 \\
    & & SD3~\cite{sd3} + FreCaS~\cite{zhang2025frecas} &79.30 &72.82 &0.0226 &0.0213 &30.32 &86 \\
    & & SD3~\cite{sd3} + HiFlow~\cite{hiflow} &71.39 &58.40 &0.0133 &0.0118 &31.59 &112 \\
    & & SD3~\cite{sd3} + Ours &\textbf{68.63} &\textbf{53.04} &\textbf{0.0089} &\textbf{0.0079} &\textbf{31.79} &\textbf{58} \\
    \cline{3-9}
    & & FLUX~\cite{blackforestlabs2024flux} (\textit{Direct}) &75.68 &72.73 &0.0192 &0.0188 &29.57 &203 \\
    & & FLUX~\cite{blackforestlabs2024flux} + DemoFusion~\cite{du2024demofusion} &71.18 &62.38 &0.0150 &0.0142 &30.48 &649 \\
    & & FLUX~\cite{blackforestlabs2024flux} + DiffuseHigh~\cite{kim2025diffusehigh} &69.23 &59.13 &0.0132 &0.0125 &30.97 &215 \\
    & & FLUX~\cite{blackforestlabs2024flux} + I-Max~\cite{Imax} &70.11 &59.83 &0.0142 &0.0133 &30.52 &484 \\
    & & FLUX~\cite{blackforestlabs2024flux} + HiFlow~\cite{hiflow} &68.36 &57.43 &0.0124 &0.0114 &30.62 &209 \\
    & & FLUX~\cite{blackforestlabs2024flux} + Ours &\textbf{64.91} &\textbf{51.84} &\textbf{0.0078} &\textbf{0.0074} &\textbf{31.17} &\textbf{116} \\
    \hline
  \end{tabular}
  }
  \label{table:1}
\end{table} 

To evaluate the efficacy of our HRDiT framework, we apply it to several popular off-the-shelf DiT models, 
\noindent including Stable Diffusion 3 (SD3)~\cite{sd3} and FLUX~\cite{blackforestlabs2024flux}. Following the settings in~\cite{hiflow,qiu2024freescale}, we conduct experiments on both 2K ($2048 \times 2048$) and 4K ($4096 \times 4096$) image generation. For text prompts used in our evaluation, following~\cite{du2024demofusion,kim2025diffusehigh}, we randomly sample 1,000 captions from the LAION-5B dataset~\cite{schuhmann2022laion}. To assess image quality, following \cite{kim2025diffusehigh}, we report evaluation metrics including Fréchet Inception Distance (FID), Kernel Inception Distance (KID), their patch-based counterparts ($\text{FID}_{p}$ and $\text{KID}_{p}$), and the CLIP score. In addition, we report the per-generation runtime to evaluate efficiency. 
We set the hyperparameters $N_{\text{scope}} = 50$, $r_c = 0.1$. 
We conduct our experiments on NVIDIA H200 GPU. More experimental settings and implementation details are in Supplementary.

\subsection{Main Results}

In our experiments, we compare against three categories of methods: (1) \textit{Direct}, which directly scales off-the-shelf DiT models to higher-resolution generation by passing them higher-resolution initial noises; 
(2) recent training-free text-to-high-resolution image generation methods applicable to DiT models, including the DiT-tailored method I-Max \cite{Imax}, and methods that focus on modifying the diffusion process of diffusion models \cite{du2024demofusion,kim2025diffusehigh,zhang2025frecas,hiflow}; 
(3) the state-of-the-art training-free text-to-high-resolution image generation method that is tailored to U-Net-based diffusion models (i.e., FreeScale \cite{qiu2024freescale}). For FreeScale, following its original setting, we apply it to the U-Net-based model Stable Diffusion XL (SDXL) \cite{sdxl}.
We present the main quantitative results in Tab.~\ref{table:1}. 
\begin{wrapfigure}[19]{r}{0.58\textwidth}
\centering
\includegraphics[width=\linewidth]{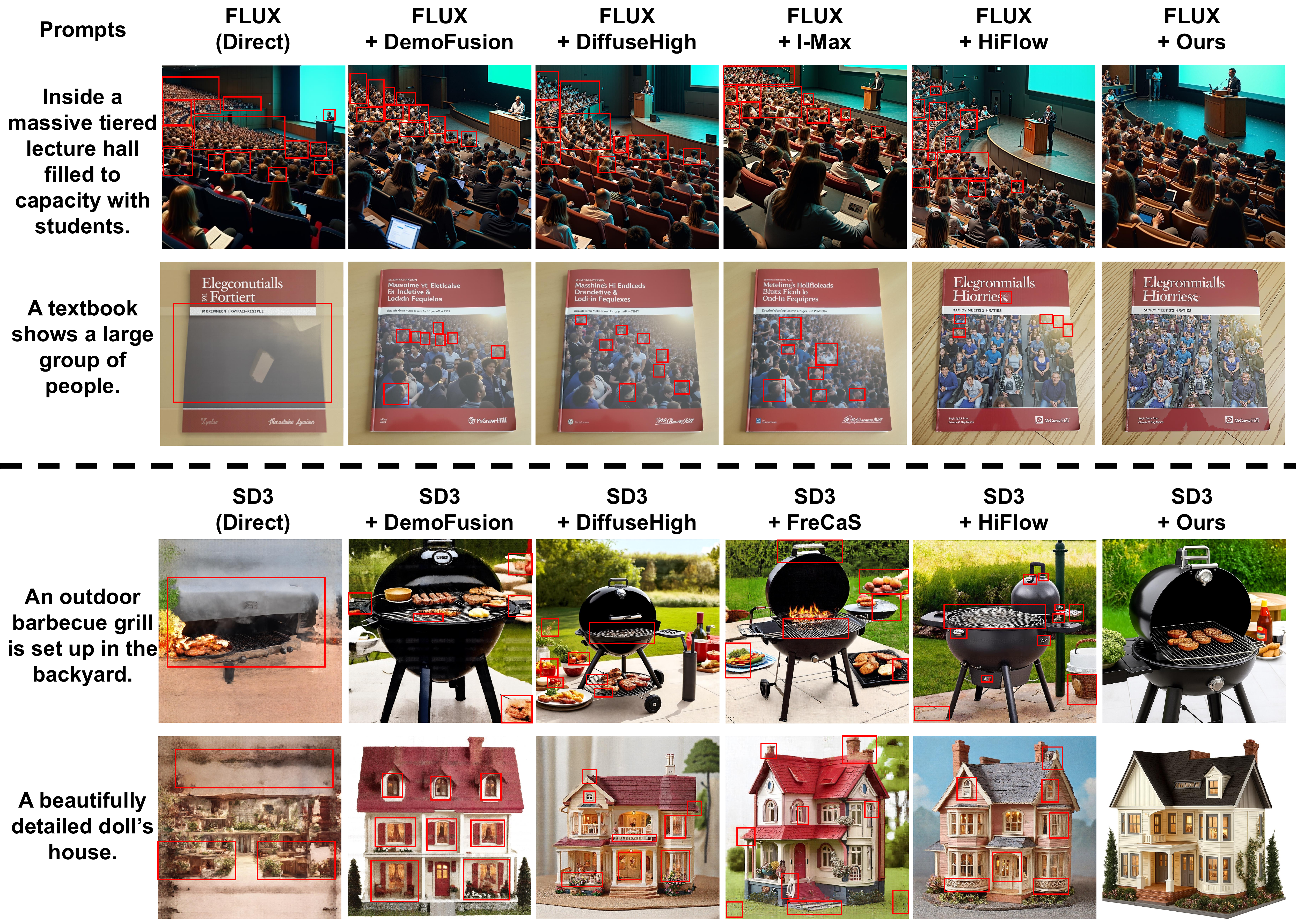}
\caption{Qualitative results of the 4K images generated by our method and by other methods (zoom-in for better view). More qualitative results are in Supplementary.}
\label{fig:fig4_qualitative}
\end{wrapfigure}
As shown, across both 2K and 4K generation tasks, our method consistently outperforms all baselines on all five metrics assessing generation quality, showing its strong ability to produce high-quality, high-resolution images. Moreover, when applied to either Stable Diffusion 3 (SD3) or FLUX, our method achieves substantial efficiency gains over all existing approaches under the same DiT backbone, validating its efficacy in addressing the long generation time challenge. Qualitative comparisons in Fig.~\ref{fig:fig4_qualitative} further confirm that existing methods often yield high-resolution images with spatial disorder (marked by red boxes), while our method preserves strong spatial coherence, further showing its efficacy.

\subsection{Additional Experiments and Ablation Studies}

\textbf{Additional experiments on 8K ($8192 \times 8192$) generation.} To assess scalability, in addition to 2K and 4K, we further evaluate our framework on 8K generation. As shown in Tab.~\ref{table:1_8k}, our method consistently achieves highest generation quality and fastest generation speed across both Stable Diffusion 3 and FLUX, further showing its efficacy at even higher resolutions. Qualitative 8K results are in Supplementary.

\begin{wraptable}[8]{r}{0.5\columnwidth}
  \caption{Evaluation on key components in HRDiT.}
  \centering
  \resizebox{\linewidth}{!}{
  \begin{tabular}{lcccccc}
    \hline
    Method & FID $\downarrow$ & FID$_\text{p}\downarrow$ & KID $\downarrow$ & KID$_\text{p}\downarrow$ & CLIP Score $\uparrow$ & Latency(s) $\downarrow$ \\
    \hline
    FLUX \cite{blackforestlabs2024flux} (\textit{Direct}) &75.68 &72.73 &0.0192 &0.0188 &29.57 &203 \\
    \hline
    Ours (w/o SPA) &75.85 &72.83 &0.0195 &0.0192 &29.55  &110 \\
    Ours (w/o HAP) &64.82 &51.81 &0.0076 &0.0074 &31.19 &212 \\
    Ours (full) &64.91 &51.84 &0.0078 &0.0074 &31.17 &116 \\
    \hline
  \end{tabular}
  }
  \label{Tab:ablation_1}
\end{wraptable}
\noindent\textbf{Impact of key components in the HRDiT framework.} In our framework (\textbf{Ours (full)}), we integrate two key components: SPA and HAP. To assess the contribution of each component, we test two variants, on 4K image generation task using FLUX as the generation model. In the first variant (\textbf{Ours (w/o SPA)}), we replace $\mathbf{c^{SPA}_{i,j}}$ with the original $\mathbf{c_{i,j}}$ in Eq.~\ref{Eq:2_5} as the per-token contribution. In the second variant (\textbf{Ours (w/o HAP)}), we do not apply the determined attention scope assignment plan to prune ``redundant'' attention computations. Instead, we direct employ the full attentions scope for all heads. As shown in Tab.~\ref{Tab:ablation_1}, our full framework significantly outperforms the first variant in generation quality. Meanwhile, it also achieves substantially faster generation speed than the second variant, with negligible degradation in generation quality. Together, these results validate the efficacy of both key components in our framework. \textbf{More ablation studies are in Supplementary.}

\begin{table}[t]
  \caption{Additional results on 8K image generation.} 
  \centering
  \resizebox{\linewidth}{!}{
  \begin{tabular}{c|clcccccc}
    \hline
    Resolution & Model Type &
    Method & FID $\downarrow$ & FID$_\text{p}\downarrow$ & KID $\downarrow$ & KID$_\text{p}\downarrow$ & CLIP Score $\uparrow$ & Latency(s) $\downarrow$ \\
    \hline
    \multirow{13}{*}{$\mathbf{8k}$} & \multirow{1}{*}{U-Net backbone} & SDXL~\cite{sdxl} + FreeScale~\cite{qiu2024freescale} &76.51 &70.50 &0.0123 &0.0176 &30.44 &3958 \\
    \cline{2-9}
     & \multirow{12}{*}{DiT-based backbone} & SD3~\cite{sd3} (\textit{Direct}) &86.75 &77.68 &0.0258 &0.0251 &30.01 &822 \\
    & & SD3~\cite{sd3} + DemoFusion~\cite{du2024demofusion} &76.07 &65.22 &0.0207 &0.0206 &30.96 &2720 \\
    & & SD3~\cite{sd3} + DiffuseHigh~\cite{kim2025diffusehigh} &73.30 &60.97 &0.0166 &0.0161 &31.19 &889 \\
    & & SD3~\cite{sd3} + FreCaS~\cite{zhang2025frecas} &81.19 &76.06 &0.0241 &0.0232 &30.16 &635 \\
    & & SD3~\cite{sd3} + HiFlow~\cite{hiflow} &73.72 &61.77 &0.0179 &0.0173 &31.44 &832 \\
    & & SD3~\cite{sd3} + Ours &\textbf{69.51} &\textbf{54.47} &\textbf{0.0096} &\textbf{0.0083} &\textbf{31.71} &\textbf{454} \\
    \cline{3-9}
    & & FLUX~\cite{blackforestlabs2024flux} (\textit{Direct}) &78.47 &74.59 &0.0212 &0.0201 &29.34 &1708 \\
    & & FLUX~\cite{blackforestlabs2024flux} + DemoFusion~\cite{du2024demofusion} &75.72 &68.71 &0.0201 &0.0187 &30.03 &5749 \\
    & & FLUX~\cite{blackforestlabs2024flux} + DiffuseHigh~\cite{kim2025diffusehigh} &71.02 &63.87 &0.0169 &0.0158 &30.69 &1835 \\
    & & FLUX~\cite{blackforestlabs2024flux} + I-Max~\cite{Imax} &72.54 &65.34 &0.0196 &0.0183 &30.23 &4488 \\
    & & FLUX~\cite{blackforestlabs2024flux} + HiFlow~\cite{hiflow} &70.67 &60.79 &0.0142 &0.0136 &30.43 &1721 \\
    & & FLUX~\cite{blackforestlabs2024flux} + Ours &\textbf{65.73} &\textbf{53.55} &\textbf{0.0089} &\textbf{0.0084} &\textbf{31.04} &\textbf{827} \\
    \hline
  \end{tabular}
  }
  \label{table:1_8k}
\end{table}
\noindent

\section{Conclusion}

In this paper, we have proposed HRDiT, a novel training-free framework for text-to-high-resolution image generation. HRDiT integrates two components, SPA and HAP, to address two key challenges, spatial disorder and long generation time, thereby enabling off-the-shelf DiT models to be effectively adapted for high-resolution synthesis.
Extensive experiments show HRDiT's efficacy.

\bibliographystyle{splncs04}
\bibliography{main}

@String(AAAI = {AAAI})

@inproceedings{zhang2025ditfastattnv2,
  title={DiTFastAttnV2: Head-wise Attention Compression for Multi-Modality Diffusion Transformers},
  author={Zhang, Hanling and Su, Rundong and Yuan, Zhihang and Chen, Pengtao and Shen, Mingzhu and Fan, Yibo and Yan, Shengen and Dai, Guohao and Wang, Yu},
  booktitle={Proceedings of the IEEE/CVF International Conference on Computer Vision},
  pages={16399--16409},
  year={2025}
}

@inproceedings{
zhang2025frecas,
title={FreCaS: Efficient Higher-Resolution Image Generation via Frequency-aware Cascaded Sampling},
author={Zhengqiang Zhang and Ruihuang Li and Lei Zhang},
booktitle={The Thirteenth International Conference on Learning Representations},
year={2025},
url={https://openreview.net/forum?id=TsBDfe8Ra5}
}

@article{haussler1992decision,
  title={Decision theoretic generalizations of the PAC model for neural net and other learning applications},
  author={Haussler, David},
  journal={Information and computation},
  volume={100},
  number={1},
  pages={78--150},
  year={1992},
  publisher={Elsevier}
}

@inproceedings{yu2024scaling,
  title={Scaling up to excellence: Practicing model scaling for photo-realistic image restoration in the wild},
  author={Yu, Fanghua and Gu, Jinjin and Li, Zheyuan and Hu, Jinfan and Kong, Xiangtao and Wang, Xintao and He, Jingwen and Qiao, Yu and Dong, Chao},
  booktitle={Proceedings of the IEEE/CVF conference on computer vision and pattern recognition},
  pages={25669--25680},
  year={2024}
}

@inproceedings{train_higen,
  title={Diffusion-4k: Ultra-high-resolution image synthesis with latent diffusion models},
  author={Zhang, Jinjin and Huang, Qiuyu and Liu, Junjie and Guo, Xiefan and Huang, Di},
  booktitle={Proceedings of the Computer Vision and Pattern Recognition Conference},
  pages={23464--23473},
  year={2025}
}

@article{train_higen2,
  title={Ultrapixel: Advancing ultra high-resolution image synthesis to new peaks},
  author={Ren, Jingjing and Li, Wenbo and Chen, Haoyu and Pei, Renjing and Shao, Bin and Guo, Yong and Peng, Long and Song, Fenglong and Zhu, Lei},
  journal={Advances in Neural Information Processing Systems},
  volume={37},
  pages={111131--111171},
  year={2024}
}

@article{hartmann2025power,
  title={The power of generative marketing: Can generative AI create superhuman visual marketing content?},
  author={Hartmann, Jochen and Exner, Yannick and Domdey, Samuel},
  journal={International Journal of Research in Marketing},
  volume={42},
  number={1},
  pages={13--31},
  year={2025},
  publisher={Elsevier}
}

@article{game,
author = {Qin, Jiayang},
year = {2023},
month = {12},
pages = {107-111},
title = {How does Text-to-image AI Affect Indie Game Designers and Artists?},
volume = {5},
journal = {Journal of Innovation and Development},
doi = {10.54097/f7of9f8k}
}

@misc{blackforestlabs2024flux,
    author={Black Forest Labs},
    title={FLUX},
    year={2024},
    howpublished={\url{https://github.com/black-forest-labs/flux}},
    note = {last accessed 2026/06/29}
}

@inproceedings{sd3,
  title={Scaling rectified flow transformers for high-resolution image synthesis},
  author={Esser, Patrick and Kulal, Sumith and Blattmann, Andreas and Entezari, Rahim and M{\"u}ller, Jonas and Saini, Harry and Levi, Yam and Lorenz, Dominik and Sauer, Axel and Boesel, Frederic and others},
  booktitle={Forty-first international conference on machine learning},
  year={2024}
}

@article{pe1,
  title={Roformer: Enhanced transformer with rotary position embedding},
  author={Su, Jianlin and Ahmed, Murtadha and Lu, Yu and Pan, Shengfeng and Bo, Wen and Liu, Yunfeng},
  journal={Neurocomputing},
  volume={568},
  pages={127063},
  year={2024},
  publisher={Elsevier}
}

@article{pe2,
  title={Attention is all you need},
  author={Vaswani, Ashish and Shazeer, Noam and Parmar, Niki and Uszkoreit, Jakob and Jones, Llion and Gomez, Aidan N and Kaiser, {\L}ukasz and Polosukhin, Illia},
  journal={Advances in neural information processing systems},
  volume={30},
  year={2017}
}

@article{sdxl,
  title={Sdxl: Improving latent diffusion models for high-resolution image synthesis},
  author={Podell, Dustin and English, Zion and Lacey, Kyle and Blattmann, Andreas and Dockhorn, Tim and M{\"u}ller, Jonas and Penna, Joe and Rombach, Robin},
  journal={arXiv preprint arXiv:2307.01952},
  year={2023}
}

@inproceedings{bar2023multidiffusion,
  title={MultiDiffusion: Fusing Diffusion Paths for Controlled Image Generation},
  author={Bar-Tal, Omer and Yariv, Lior and Lipman, Yaron and Dekel, Tali},
  booktitle={International Conference on Machine Learning},
  pages={1737--1752},
  year={2023},
  organization={PMLR}
}

@inproceedings{du2024demofusion,
  title={Demofusion: Democratising high-resolution image generation with no \$\$\$},
  author={Du, Ruoyi and Chang, Dongliang and Hospedales, Timothy and Song, Yi-Zhe and Ma, Zhanyu},
  booktitle={Proceedings of the IEEE/CVF conference on computer vision and pattern recognition},
  pages={6159--6168},
  year={2024}
}

@article{hiflow,
  title={HiFlow: Training-free High-Resolution Image Generation with Flow-Aligned Guidance},
  author={Bu, Jiazi and Ling, Pengyang and Zhou, Yujie and Zhang, Pan and Wu, Tong and Dong, Xiaoyi and Zang, Yuhang and Cao, Yuhang and Lin, Dahua and Wang, Jiaqi},
  journal={arXiv preprint arXiv:2504.06232},
  year={2025}
}

@article{Imax,
  title={I-Max: Maximize the Resolution Potential of Pre-trained Rectified Flow Transformers with Projected Flow},
  author={Du, Ruoyi and Liu, Dongyang and Zhuo, Le and Qi, Qin and Li, Hongsheng and Ma, Zhanyu and Gao, Peng},
  journal={arXiv preprint arXiv:2410.07536},
  year={2024}
}

@article{han2023lm,
  title={Lm-infinite: Zero-shot extreme length generalization for large language models},
  author={Han, Chi and Wang, Qifan and Peng, Hao and Xiong, Wenhan and Chen, Yu and Ji, Heng and Wang, Sinong},
  journal={arXiv preprint arXiv:2308.16137},
  year={2023}
}

@inproceedings{guo2024make,
  title={Make a cheap scaling: A self-cascade diffusion model for higher-resolution adaptation},
  author={Guo, Lanqing and He, Yingqing and Chen, Haoxin and Xia, Menghan and Cun, Xiaodong and Wang, Yufei and Huang, Siyu and Zhang, Yong and Wang, Xintao and Chen, Qifeng and others},
  booktitle={European conference on computer vision},
  pages={39--55},
  year={2024},
  organization={Springer}
}

@inproceedings{hoogeboom2023simple,
  title={simple diffusion: End-to-end diffusion for high resolution images},
  author={Hoogeboom, Emiel and Heek, Jonathan and Salimans, Tim},
  booktitle={International Conference on Machine Learning},
  pages={13213--13232},
  year={2023},
  organization={PMLR}
}

@inproceedings{he2023scalecrafter,
  title={Scalecrafter: Tuning-free higher-resolution visual generation with diffusion models},
  author={He, Yingqing and Yang, Shaoshu and Chen, Haoxin and Cun, Xiaodong and Xia, Menghan and Zhang, Yong and Wang, Xintao and He, Ran and Chen, Qifeng and Shan, Ying},
  booktitle={The Twelfth International Conference on Learning Representations},
  year={2023}
}

@inproceedings{lin2024accdiffusion,
  title={Accdiffusion: An accurate method for higher-resolution image generation},
  author={Lin, Zhihang and Lin, Mingbao and Zhao, Meng and Ji, Rongrong},
  booktitle={European Conference on Computer Vision},
  pages={38--53},
  year={2024},
  organization={Springer}
}

@article{qiu2024freescale,
  title={Freescale: Unleashing the resolution of diffusion models via tuning-free scale fusion},
  author={Qiu, Haonan and Zhang, Shiwei and Wei, Yujie and Chu, Ruihang and Yuan, Hangjie and Wang, Xiang and Zhang, Yingya and Liu, Ziwei},
  journal={arXiv preprint arXiv:2412.09626},
  year={2024}
}

@inproceedings{zhang2024hidiffusion,
  title={Hidiffusion: Unlocking higher-resolution creativity and efficiency in pretrained diffusion models},
  author={Zhang, Shen and Chen, Zhaowei and Zhao, Zhenyu and Chen, Yuhao and Tang, Yao and Liang, Jiajun},
  booktitle={European Conference on Computer Vision},
  pages={145--161},
  year={2024},
  organization={Springer}
}

@inproceedings{yang2025fam,
  title={Fam diffusion: Frequency and attention modulation for high-resolution image generation with stable diffusion},
  author={Yang, Haosen and Bulat, Adrian and Hadji, Isma and Pham, Hai X and Zhu, Xiatian and Tzimiropoulos, Georgios and Martinez, Brais},
  booktitle={Proceedings of the Computer Vision and Pattern Recognition Conference},
  pages={2459--2468},
  year={2025}
}

@inproceedings{peebles2023scalable,
  title={Scalable diffusion models with transformers},
  author={Peebles, William and Xie, Saining},
  booktitle={Proceedings of the IEEE/CVF international conference on computer vision},
  pages={4195--4205},
  year={2023}
}

@inproceedings{
zhang2025fast,
title={Fast Video Generation with Sliding Tile Attention},
author={Yongqi Chen and Peiyuan Zhang and Runlong Su and Hangliang Ding and Ion Stoica and Zhengzhong Liu and Hao Zhang},
booktitle={Forty-second International Conference on Machine Learning},
year={2025},
url={https://openreview.net/forum?id=U74MOXPEJd}
}

@inproceedings{
wu2025retrieval,
title={Retrieval Head Mechanistically Explains Long-Context Factuality},
author={Wenhao Wu and Yizhong Wang and Guangxuan Xiao and Hao Peng and Yao Fu},
booktitle={The Thirteenth International Conference on Learning Representations},
year={2025},
url={https://openreview.net/forum?id=EytBpUGB1Z}
}

@manual{gurobi,
  title        = {Gurobi Optimizer Reference Manual},
  author       = {Gurobi Optimization, LLC},
  year         = {2025},
  url          = {https://www.gurobi.com},
  note = {last accessed 2026/06/29}
}

@article{schuhmann2022laion,
  title={Laion-5b: An open large-scale dataset for training next generation image-text models},
  author={Schuhmann, Christoph and Beaumont, Romain and Vencu, Richard and Gordon, Cade and Wightman, Ross and Cherti, Mehdi and Coombes, Theo and Katta, Aarush and Mullis, Clayton and Wortsman, Mitchell and others},
  journal={Advances in neural information processing systems},
  volume={35},
  pages={25278--25294},
  year={2022}
}

@inproceedings{kim2025diffusehigh,
  title={Diffusehigh: Training-free progressive high-resolution image synthesis through structure guidance},
  author={Kim, Younghyun and Hwang, Geunmin and Zhang, Junyu and Park, Eunbyung},
  booktitle={Proceedings of the AAAI conference on artificial intelligence},
  volume={39},
  number={4},
  pages={4338--4346},
  year={2025}
}

@inproceedings{huang2024fouriscale,
  title={Fouriscale: A frequency perspective on training-free high-resolution image synthesis},
  author={Huang, Linjiang and Fang, Rongyao and Zhang, Aiping and Song, Guanglu and Liu, Si and Liu, Yu and Li, Hongsheng},
  booktitle={European conference on computer vision},
  pages={196--212},
  year={2024},
  organization={Springer}
}

@article{lee2024groundit,
  title={Groundit: Grounding diffusion transformers via noisy patch transplantation},
  author={Lee, Yuseung and Yoon, Taehoon and Sung, Minhyuk},
  journal={Advances in Neural Information Processing Systems},
  volume={37},
  pages={58610--58636},
  year={2024}
}

@inproceedings{wu2025megafusion,
  title={Megafusion: Extend diffusion models towards higher-resolution image generation without further tuning},
  author={Wu, Haoning and Shen, Shaocheng and Hu, Qiang and Zhang, Xiaoyun and Zhang, Ya and Wang, Yanfeng},
  booktitle={2025 IEEE/CVF Winter Conference on Applications of Computer Vision (WACV)},
  pages={3944--3953},
  year={2025},
  organization={IEEE}
}

@article{liu2025boosting,
  title={Boosting Resolution Generalization of Diffusion Transformers with Randomized Positional Encodings},
  author={Liu, Cong and Hou, Liang and Zheng, Mingwu and Tao, Xin and Wan, Pengfei and Zhang, Di and Gai, Kun},
  journal={arXiv preprint arXiv:2503.18719},
  year={2025}
}

@article{zhang2025ledit,
  title={LEDiT: Your Length-Extrapolatable Diffusion Transformer without Positional Encoding},
  author={Zhang, Shen and Liang, Siyuan and Tan, Yaning and Chen, Zhaowei and Li, Linze and Wu, Ge and Chen, Yuhao and Li, Shuheng and Zhao, Zhenyu and Chen, Caihua and others},
  journal={arXiv preprint arXiv:2503.04344},
  year={2025}
}

@article{qiu2025cinescale,
  title={CineScale: Free Lunch in High-Resolution Cinematic Visual Generation},
  author={Qiu, Haonan and Yu, Ning and Huang, Ziqi and Debevec, Paul and Liu, Ziwei},
  journal={arXiv preprint arXiv:2508.15774},
  year={2025}
}

@article{yang2025rectifiedhr,
  title={Rectifiedhr: Enable efficient high-resolution image generation via energy rectification},
  author={Yang, Zhen and Shen, Guibao and Hou, Liang and Liu, Mushui and Wang, Luozhou and Tao, Xin and Wan, Pengfei and Zhang, Di and Chen, Ying-Cong},
  journal={arXiv e-prints},
  pages={arXiv--2503},
  year={2025}
}

@article{liu2024clear,
  title={Clear: Conv-like linearization revs pre-trained diffusion transformers up},
  author={Liu, Songhua and Tan, Zhenxiong and Wang, Xinchao},
  journal={arXiv preprint arXiv:2412.16112},
  year={2024}
}

@inproceedings{tang2022patch,
  title={Patch slimming for efficient vision transformers},
  author={Tang, Yehui and Han, Kai and Wang, Yunhe and Xu, Chang and Guo, Jianyuan and Xu, Chao and Tao, Dacheng},
  booktitle={Proceedings of the IEEE/CVF Conference on Computer Vision and Pattern Recognition},
  pages={12165--12174},
  year={2022}
}

@article{bai2025positional,
  title={Positional Encoding Field},
  author={Bai, Yunpeng and Li, Haoxiang and Huang, Qixing},
  journal={arXiv preprint arXiv:2510.20385},
  year={2025}
}

@article{shen2025mmdit,
  title={E-MMDiT: Revisiting Multimodal Diffusion Transformer Design for Fast Image Synthesis under Limited Resources},
  author={Shen, Tong and Yu, Jingai and Zhou, Dong and Li, Dong and Barsoum, Emad},
  journal={arXiv preprint arXiv:2510.27135},
  year={2025}
}

@article{fu2024moa,
  title={Moa: Mixture of sparse attention for automatic large language model compression},
  author={Zhang, Genghan and Fu, Tianyu and Huang, Haofeng and Ning, Xuefei and  and Chen, Boju and Wu, Tianqi and Wang, Hongyi and Huang, Zixiao and Li, Shiyao and Yan, Shengen and others},
  journal={arXiv preprint arXiv:2406.14909},
  year={2024}
}

@article{zheng2022savit,
  title={Savit: Structure-aware vision transformer pruning via collaborative optimization},
  author={Zheng, Chuanyang and Zhang, Kai and Yang, Zhi and Tan, Wenming and Xiao, Jun and Ren, Ye and Pu, Shiliang and others},
  journal={Advances in Neural Information Processing Systems},
  volume={35},
  pages={9010--9023},
  year={2022}
}

@inproceedings{
leroy2024winwin,
title={Win-Win: Training High-Resolution Vision Transformers from Two Windows},
author={Vincent Leroy and Jerome Revaud and Thomas Lucas and Philippe Weinzaepfel},
booktitle={The Twelfth International Conference on Learning Representations},
year={2024},
url={https://openreview.net/forum?id=N23A4ybMJr}
}

@article{yuan2024ditfastattn,
  title={Ditfastattn: Attention compression for diffusion transformer models},
  author={Yuan, Zhihang and Zhang, Hanling and Pu, Lu and Ning, Xuefei and Zhang, Linfeng and Zhao, Tianchen and Yan, Shengen and Dai, Guohao and Wang, Yu},
  journal={Advances in Neural Information Processing Systems},
  volume={37},
  pages={1196--1219},
  year={2024}
}

@article{cao2024ap,
  title={Ap-ldm: Attentive and progressive latent diffusion model for training-free high-resolution image generation},
  author={Cao, Boyuan and Ye, Jiaxin and Wei, Yujie and Shan, Hongming},
  journal={arXiv preprint arXiv:2410.06055},
  year={2024}
}

@inproceedings{
ye2025flashinfer,
title={FlashInfer: Efficient and Customizable Attention Engine for {LLM} Inference Serving},
author={Zihao Ye and Lequn Chen and Ruihang Lai and Wuwei Lin and Yineng Zhang and Stephanie Wang and Tianqi Chen and Baris Kasikci and Vinod Grover and Arvind Krishnamurthy and Luis Ceze},
booktitle={Eighth Conference on Machine Learning and Systems},
year={2025},
url={https://openreview.net/forum?id=RXPofAsL8F}
}

@inproceedings{li2025longdiff,
  title={LongDiff: Training-Free Long Video Generation in One Go},
  author={Li, Zhuoling and Rahmani, Hossein and Ke, Qiuhong and Liu, Jun},
  booktitle={Proceedings of the Computer Vision and Pattern Recognition Conference},
  pages={17789--17798},
  year={2025}
}

@article{jin2023training,
  title={Training-free diffusion model adaptation for variable-sized text-to-image synthesis},
  author={Jin, Zhiyu and Shen, Xuli and Li, Bin and Xue, Xiangyang},
  journal={Advances in Neural Information Processing Systems},
  volume={36},
  pages={70847--70860},
  year={2023}
}

@article{lai2026pixelrush,
  title={PixelRush: Ultra-Fast, Training-Free High-Resolution Image Generation via One-step Diffusion},
  author={Lai, Hong-Phuc and Nguyen, Phong and Tran, Anh},
  journal={arXiv preprint arXiv:2602.12769},
  year={2026}
}

@inproceedings{vontobel2025hiwave,
  title={HiWave: Training-Free High-Resolution Image Generation via Wavelet-Based Diffusion Sampling},
  author={Vontobel, Tobias and Sadat, Seyedmorteza and Salehi, Farnood and Weber, Romann},
  booktitle={Proceedings of the SIGGRAPH Asia 2025 Conference Papers},
  pages={1--11},
  year={2025}
}

@article{zhao2025ultraimage,
  title={UltraImage: Rethinking Resolution Extrapolation in Image Diffusion Transformers},
  author={Zhao, Min and Yan, Bokai and Yang, Xue and Zhu, Hongzhou and Zhang, Jintao and Liu, Shilong and Li, Chongxuan and Zhu, Jun},
  journal={arXiv preprint arXiv:2512.04504},
  year={2025}
}

@article{koh2025scalediff,
  title={ScaleDiff: Higher-Resolution Image Synthesis via Efficient and Model-Agnostic Diffusion},
  author={Koh, Sungho and Cha, SeungJu and Oh, Hyunwoo and Lee, Kwanyoung and Kim, Dong-Jin},
  journal={arXiv preprint arXiv:2510.25818},
  year={2025}
}

@article{issachar2025dype,
  title={DyPE: Dynamic Position Extrapolation for Ultra High Resolution Diffusion},
  author={Issachar, Noam and Yariv, Guy and Benaim, Sagie and Adi, Yossi and Lischinski, Dani and Fattal, Raanan},
  journal={arXiv preprint arXiv:2510.20766},
  year={2025}
}

@inproceedings{
jin2024llm,
title={{LLM} Maybe Long{LM}: SelfExtend {LLM} Context Window Without Tuning},
author={Hongye Jin and Xiaotian Han and Jingfeng Yang and Zhimeng Jiang and Zirui Liu and Chia-Yuan Chang and Huiyuan Chen and Xia Hu},
booktitle={Forty-first International Conference on Machine Learning},
year={2024},
url={https://openreview.net/forum?id=nkOMLBIiI7}
}

\clearpage
\appendix

\suppressfloats[t]
\begin{center}
  {\Large\bfseries HRDiT: Training-Free High-Resolution Image Generation
   with Off-the-Shelf Diffusion Transformer Models\par}
  \medskip
  {\Large\bfseries (Supplementary Material)\par}
\end{center}
\vspace{1.5em}

\section{More Qualitative Results}

In this section, we present additional qualitative results. Specifically, Fig.~\ref{fig:4k_comparison} provides enlarged versions of the 4K qualitative results from Fig.~4 in the main paper for improved visibility. Fig.~\ref{fig:2k_comparison} and Fig.~\ref{fig:8k_comparison} further present 2K and 8K qualitative comparisons, respectively. As illustrated, across all high-resolution settings, images generated by existing methods consistently exhibit spatial disorder (highlighted by red boxes).
In contrast, images generated by our framework consistently preserve strong spatial coherence. This further demonstrates the effectiveness of our framework.

\begin{figure}[t]
\centering
\includegraphics[width=\linewidth]{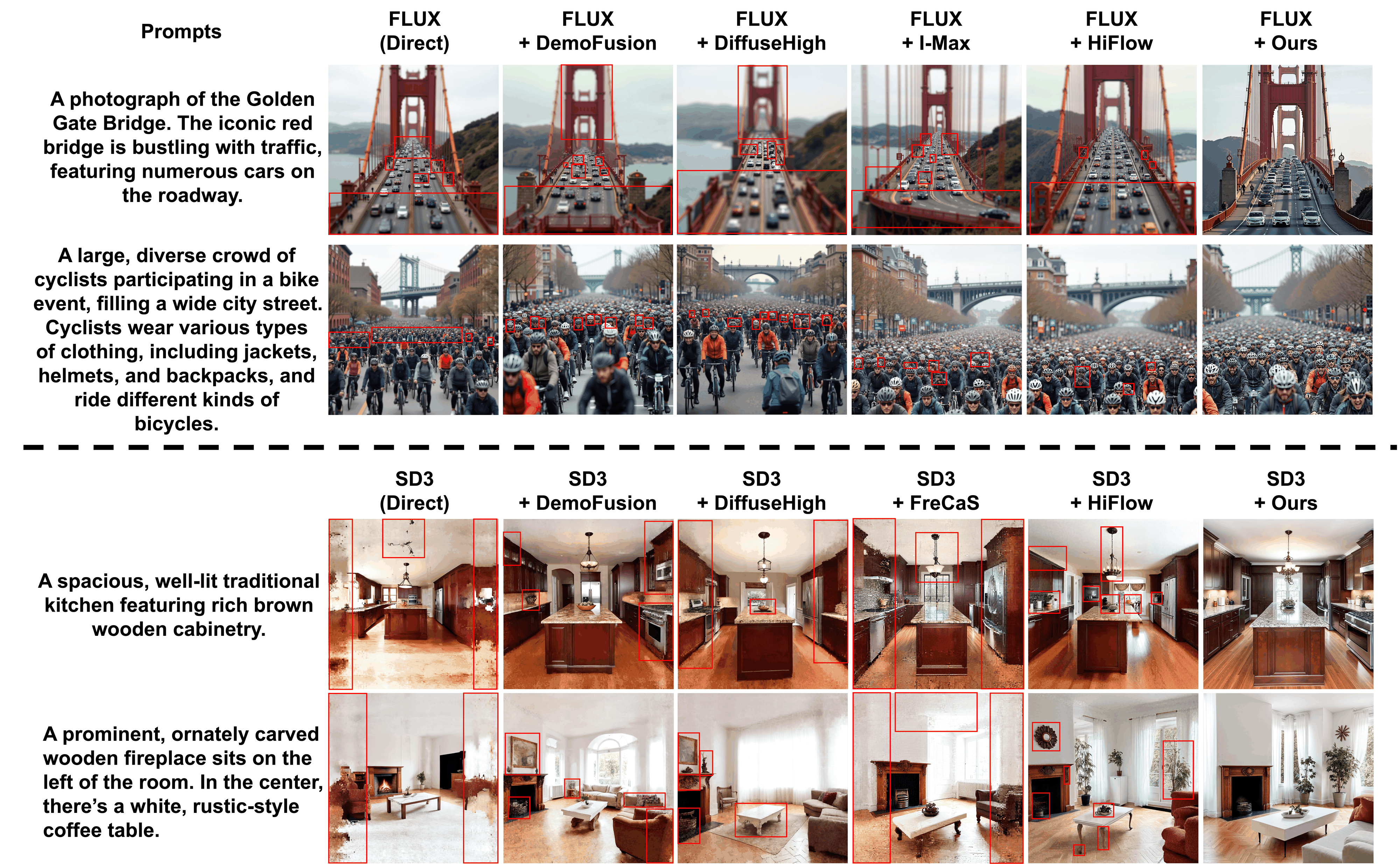}
\caption{Qualitative results of the 2K images generated by our method and by other state-of-the-art methods when applied under the same DiT backbone.}
\label{fig:2k_comparison}
\end{figure}

\begin{figure}[t]
\centering
\includegraphics[width=\linewidth]{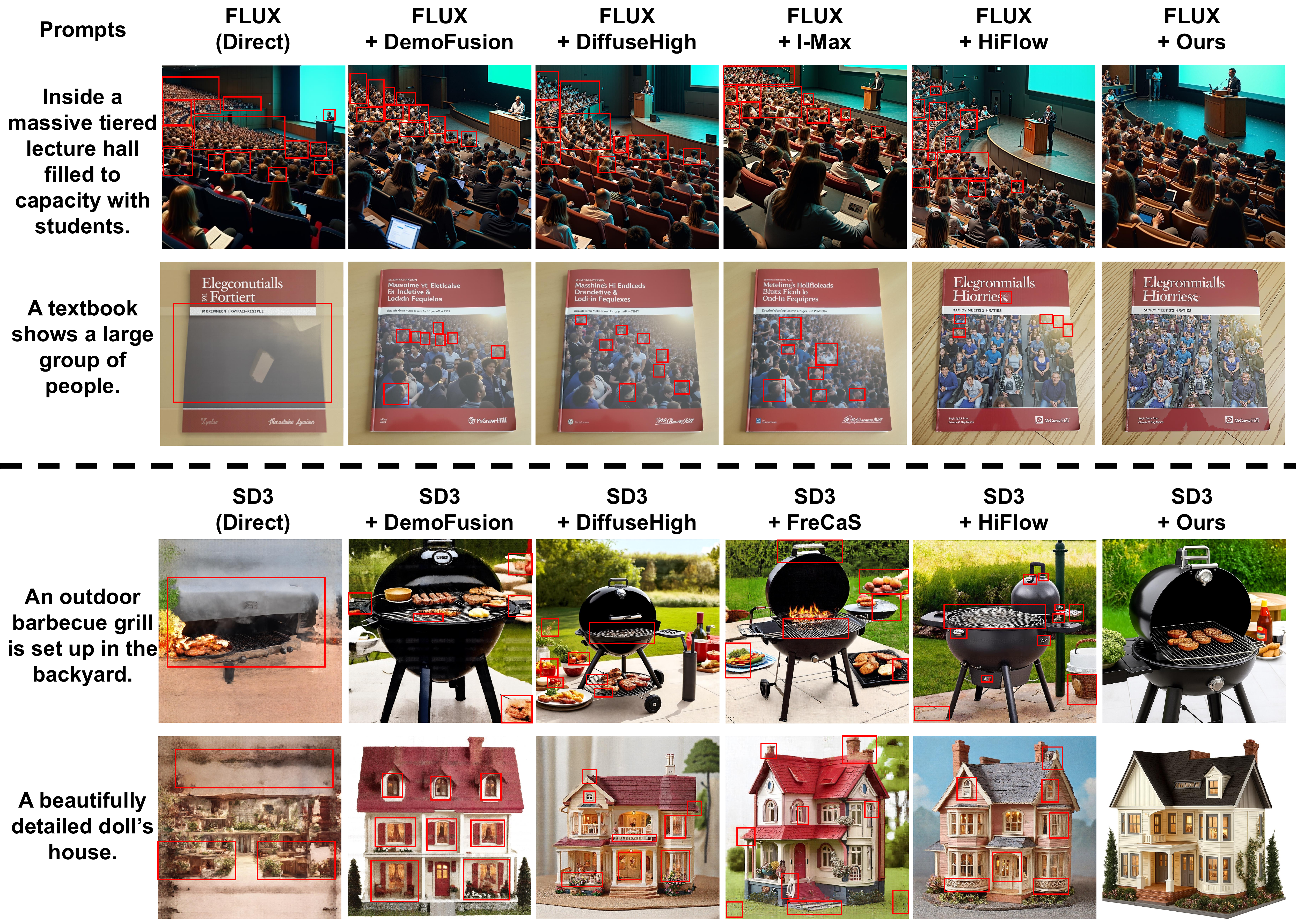}
\caption{Qualitative results of the 4K images generated by our method and by other state-of-the-art methods when applied under the same DiT backbone.}
\label{fig:4k_comparison}
\end{figure}

\begin{figure}[t]
\centering
\includegraphics[width=\linewidth]{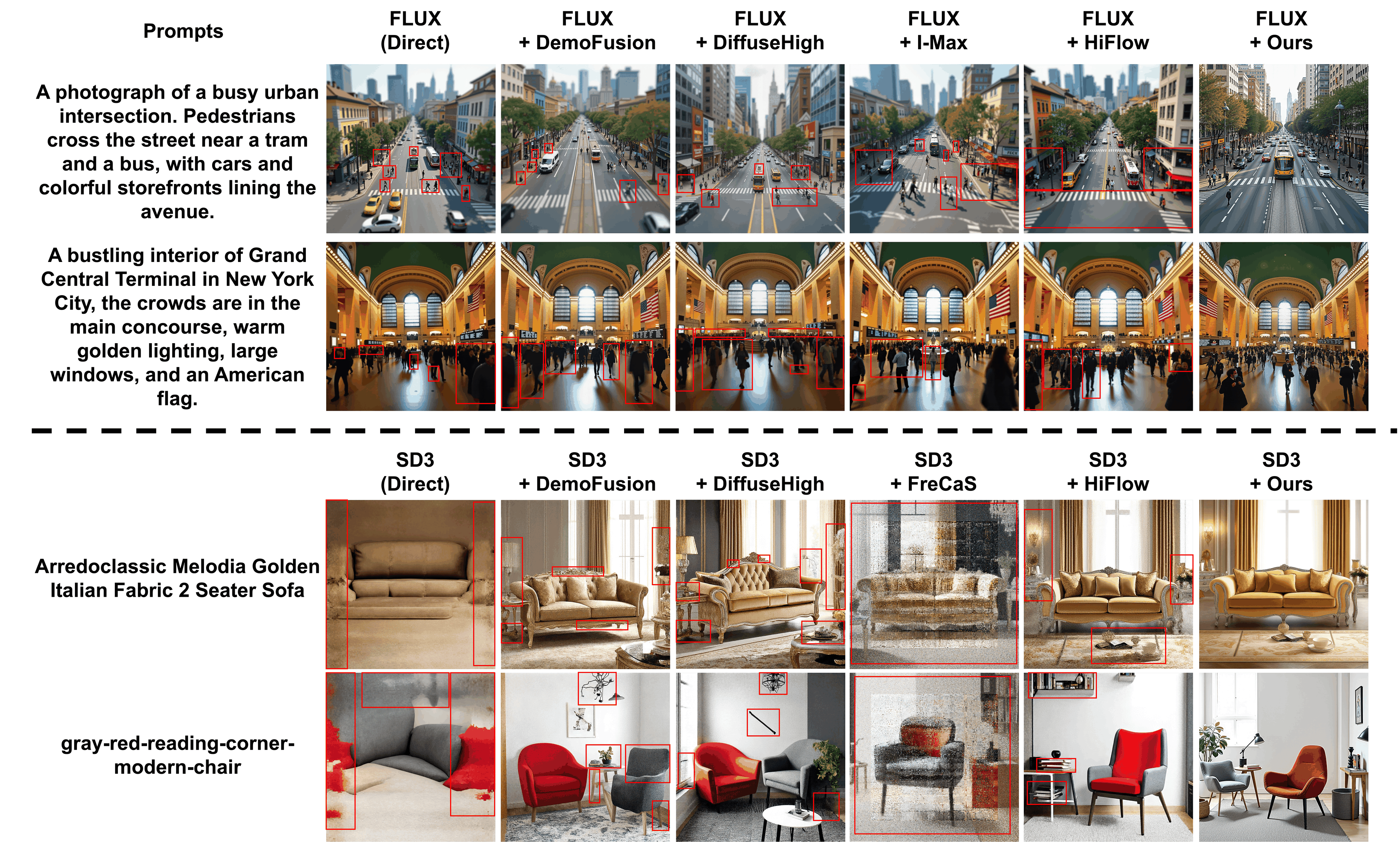}
\caption{Qualitative results of the 8K images generated by our method and by other state-of-the-art methods when applied under the same DiT backbone.}
\label{fig:8k_comparison}
\end{figure}

\section{Additional Experimental Settings and Implementation Details}
\label{sec:appendix_more_details}

In our experiments, we use the FLUX.1 dev model for FLUX~\cite{blackforestlabs2024flux} and the Stable Diffusion 3 Medium model for Stable Diffusion 3~\cite{sd3}. 
For evaluation, when computing the Fréchet Inception Distance (FID), the Kernel Inception Distance (KID), and their patch-based variants ($\text{FID}_{p}$ and $\text{KID}_{p}$), we follow~\cite{kim2025diffusehigh} and compute all four metrics with the help of 10,000 real images sampled from LAION-5B~\cite{schuhmann2022laion}. 
During computing $\text{FID}_{p}$ and $\text{KID}_{p}$, we further follow~\cite{lin2024accdiffusion} and randomly crop ten local patches from each generated image. 
In addition, in the main paper, we present the bundle and slide operations in 1D for clarity. In practice, for images whose latent tokens lie on a 2D grid, we apply the 1D bundle and slide operations independently along the two spatial dimensions. For each dimension, we set $N = 3$ for 2K image generation and $N = 5$ for 4K image generation.

During the estimation of $I_q$ in the HAP component, (1) we stabilize the process by averaging $I_q$ over 30 text prompts. 
These 30 prompts are randomly sampled from the LAION-5B dataset~\cite{schuhmann2022laion}, while explicitly excluding any prompts used during evaluation.
(2) When determining the appropriate collective scope assignment for a given testing resolution, we compute $L$ using the output from the final denoising step at that resolution. 

Besides, following prior work~\cite{kim2025diffusehigh, Imax, hiflow}, we adopt a cascaded generation strategy for synthesizing images at 1K ($1024 \times 1024$), 2K ($2048 \times 2048$), 4K ($4096 \times 4096$), and 8K ($8192 \times 8192$) resolutions. For FLUX, the sampling steps at these resolutions are 30, 16, 10, and 8, respectively; for Stable Diffusion~3, the corresponding steps are 28, 8, 6, and 5.
Following prior work~\cite{hiflow}, we further apply a commonly used technique: balancing the entropy shift of self-attention~\cite{jin2023training}. For a fair comparison, all baseline methods in our experiments adopt the same cascaded generation strategy and incorporate this technique.

\section{Additional Ablation Studies}

We here conduct more ablation experiments on the 4K image generation task using FLUX as the generation model.

\noindent\textbf{More experiments on the SPA component.} In our framework, inspired by \cite{li2025longdiff}, we adopt the SPA component to address the \textit{spatial disorder} challenge. To validate it in our context, we examine three variants. In the first variant (\textbf{w/o whole SPA}), we remove the entire SPA component from our framework. 
In the second variant (\textbf{down-scale}), instead of handling spatial disorder through SPA, we downscale token indices to match the token-index range corresponding to the model’s training resolution. Specifically, letting $T$ denote the total number of tokens at the testing resolution and $T_{\text{train}}$ the total number of tokens at the training resolution of the off-the-shelf DiT model, in this variant, we use $f_{pe}(i \times \frac{T_{\text{train}}}{T}, j \times \frac{T_{\text{train}}}{T})$ in replacement of $f_{pe}(i, j)$. 
In the third variant (\textbf{w/o slide operation}), we apply only the bundle operation of SPA and discard the slide operation. As shown in \cref{table: ablation_oper_in_spa}, both the first variant and the third variant underperform our full method, demonstrating the importance of both operations in the SPA component. Moreover, the inferior performance of the second variant implies that, merely re-fitting token indices into the pre-training token index range does not effectively handle spatial disorder. This may be because downscaling does not reduce $|S_{pe}|$. Hence, it can fail to facilitate satisfying $h(T) = |S_{pe}|$, which, according to the analysis associated with Proposition~1 in the main paper, is crucial for avoiding spatial disorder in the generated image.

\begin{table}[h]
  \caption{More experiments on the SPA component.}
  \centering
  \resizebox{0.65\linewidth}{!}{
  \begin{tabular}{lcccccc}
    \toprule
    Method & FID $\downarrow$ & FID$_\text{p}\downarrow$ & KID $\downarrow$ & KID$_\text{p}\downarrow$ & CLIP Score $\uparrow$ & Latency(s) $\downarrow$ \\
    \hline
    w/o whole SPA &75.85 &72.83 &0.0195 &0.0192 &29.55  &110 \\
    Down-scale &74.87 &71.68 &0.0186 &0.0174 &29.60 &111 \\
    w/o slide operation &69.43 &63.08 &0.0130 &0.0129 &30.52 &111 \\
    Ours &64.91 &51.84 &0.0078 &0.0074 &31.17 &116 \\
    \hline
  \end{tabular}
  }
  \label{table: ablation_oper_in_spa}
\end{table}

\noindent\textbf{More experiments on the HAP component.} Inspired by \cite{fu2024moa}, we also integrate an HAP component into our framework to address the \textit{long generation time} problem by effectively pruning redundant multi-head attention computations. To validate the efficacy of this design in our context, we examine four variants. 
In the first variant (\textbf{w/o whole HAP}), we remove the entire HAP component from our framework. 
In the second variant (\textbf{fixed local window}), instead of applying HAP, we assign a fixed local attention window to every head. For a fair comparison, we also set $r_c = 0.1$ in this variant. 
In the third variant (\textbf{random window}), we randomly assign each head an attention window such that the total computational cost of all multi-head attention operations, measured relative to that under full attention, is kept at $0.1$ (i.e., $r_c = 0.1$). 
In the fourth variant (\textbf{intermediate-feature-losses-guided window}), when determining the attention scope (window) for each head, we follow \cite{zhang2025ditfastattnv2} and estimate $I_q(n_{\text{head}}, n_{\text{scope}})$ using intermediate-feature losses instead of based on final-generation quality degradation. We also set $r_c = 0.1$ in this variant. 
As shown in \cref{table: ablation_attn_variants}, although the second, third, and fourth variants, as well as our framework, all reduce generation time compared to the first variant, all three variants suffer from noticeable quality degradation. 
In contrast, our framework equipped with the HAP component, after performing the one-time preparatory stage, achieves substantial efficiency gains during per-input generation while preserving high generation quality. This demonstrates the effectiveness of the HAP component in our context.

\begin{table}[h]
  \caption{More experiments on the HAP component.}
  \centering
  \resizebox{0.8\linewidth}{!}{
  \begin{tabular}{lcccccc}
    \toprule
    Method & FID $\downarrow$ & FID$_\text{p}\downarrow$ & KID $\downarrow$ & KID$_\text{p}\downarrow$ & CLIP Score $\uparrow$ & Latency(s) $\downarrow$ \\
    \hline
    w/o whole HAP &64.82 &51.81 &0.0076 &0.0074 &31.19 &212 \\
    \hline
    Fixed local window &69.47 &60.77 &0.0116 &0.0101 &30.83 &115\\
    Random window &73.45 &70.90 &0.0168 &0.0162 &29.79 &117 \\
    Intermediate-feature-losses-guided window &68.02 &58.21 &0.0104 &0.0093 &30.92 &116 \\
    Ours &64.91 &51.84 &0.0078 &0.0074 &31.17 &116 \\
    \bottomrule
  \end{tabular}
  }
  \label{table: ablation_attn_variants}
\end{table}

\noindent\textbf{User study.} Besides relying on quantitative evaluation metrics, we also conduct a user study to assess the generated high-resolution images based on human subjective judgment. In this study, five participants were each shown 4K images generated by all compared methods using FLUX as the underlying generation model. For each method, we produced 100 images using the same set of 100 text prompts randomly sampled from the LAION-5B dataset~\cite{schuhmann2022laion}. All images from different methods were displayed in a random order to avoid potential bias. Each participant was then asked to provide two ratings for every image, each on a 1–5 scale, evaluating prompt-image alignment and visual fidelity, respectively. The average scores for each method are summarized in \cref{tab:user_study}. As shown, our method obtains the highest ratings across both evaluation criteria.

\begin{table}[h]
\centering
\caption{Comparison based on user study.}
\resizebox{0.55\linewidth}{!}{
\begin{tabular}{lccc}
    \toprule
    Method & Prompt-Image Alignment & Visual fidelity \\
    \midrule
    FLUX~\cite{blackforestlabs2024flux} (\textit{Direct}) &2.2 & 1.8 \\
    FLUX~\cite{blackforestlabs2024flux} + DemoFusion~\cite{du2024demofusion} &3.2 & 2.4 \\
    FLUX~\cite{blackforestlabs2024flux} + DiffuseHigh~\cite{kim2025diffusehigh} &3.6 &3.7 \\
    FLUX~\cite{blackforestlabs2024flux} + I-Max~\cite{Imax} &3.3 & 2.6 \\
    FLUX~\cite{blackforestlabs2024flux} + HiFlow~\cite{hiflow} &3.8 & 3.7 \\
    \hline
    FLUX~\cite{blackforestlabs2024flux} + Ours &4.7 &4.8 \\
    \bottomrule
\end{tabular}}
\label{tab:user_study}
\end{table}

\noindent\textbf{Preparation time.} For testing at the high resolution, before testing, our framework first performs a one-time preparatory stage, which consists of two steps, to determine an appropriate collective scope assignment plan across different attention heads. We here report the runtime of this preparatory stage for the 4K image generation task using FLUX as the generation model, evaluated on an H200 GPU. Specifically, in this setting, the first step of the preparatory stage takes approximately 45 minutes, while the second step requires around 3 minutes, resulting in a total preparatory time of about 48 minutes (see Tab.~\ref{Tab:t_pre}). Importantly, this one-time process enables our framework to substantially reduce the per-input generation time, as shown in Tab.~2 of the main paper. We also note that the preparatory time becomes even shorter when using the more lightweight Stable Diffusion~3 model or when generating 2K images, which naturally require less computation than 4K. Also note that, the entire preparatory stage of our framework can be executed automatically via a single script.

\begin{table}[h]
  \caption{Preparation time of our framework. Note that the whole preparatory stage can be automatically conducted with a script.}
  \centering
  
  \resizebox{0.3\linewidth}{!}{
  \begin{tabular}{l|c}
    \hline
    & Time \\
    \hline
    Preparatory step 1 & $\sim$45 min\\
    Preparatory step 2 & $\sim$3 min\\
    Whole preparatory stage & $\sim$48 min\\
    \hline
  \end{tabular}
  }
  \label{Tab:t_pre}
\end{table}

\noindent\textbf{Impact of the number of candidate scopes $N_{\text{scope}}$.} In our framework, during performing the preparatory stage, we set the number of candidate scopes $N_{\text{scope}}$ to be 50. Here we evaluate other choices of $N_{\text{scope}}$, and report the results in \cref{table: ablation_diff_candidates}. 
As shown, starting from $N_{\text{scope}} = 10$, the generation quality improves as $N_{\text{scope}}$ increases, and then stabilizes at 50.
Therefore, we set the number of candidate scopes to be 50 in our experiments. 

\begin{table}[h]
  \caption{Evaluation on the number of candidate scopes $N_{\text{scope}}$.}
  \centering
  
  \resizebox{0.65\linewidth}{!}{
  \begin{tabular}{lccccccc}
    \toprule
    Method & FID $\downarrow$ & FID$_\text{p}\downarrow$ & KID $\downarrow$ & KID$_\text{p}\downarrow$ & CLIP Score $\uparrow$ & Latency(s) $\downarrow$\\
    \hline
     $N_{\text{scope}} = 10$ &66.03 &55.72 &0.0093 &0.0084 &31.01 &115 \\
     $N_{\text{scope}} = 30$ &65.17 &53.29 &0.0082 &0.0079 &31.13 &115 \\
     $N_{\text{scope}} = 50$ &64.91 &51.84 &0.0078 &0.0074 &31.17 &116 \\
     $N_{\text{scope}} = 70$ &64.89 &51.84 &0.0078 &0.0074 &31.18 &116 \\
     $N_{\text{scope}} = 90$ &64.88 &51.83 &0.0078 &0.0074 &31.18 &116 \\
    \hline
  \end{tabular}
  }
  \label{table: ablation_diff_candidates}
\end{table}

\noindent\textbf{Impact of the number of tokens in each middle bundle $N$.} In the SPA component of our framework, we set the number of tokens in each middle bundle $N$ to be 5 during 4K image generation. Here we also assess other choices of $N$ from 3 to 7. As shown in \cref{table: ablation_diff_bundle_size}, our framework gets optimal generation quality when $N$ is set to 5 or 6, and $N = 5$ is used in our 4K image generation experiments. 

\begin{table}[h]
  \caption{Evaluation on the number of tokens in each middle bundle $N$.}
  \centering
  \resizebox{0.6\linewidth}{!}{
  \begin{tabular}{lccccccc}
    \toprule
    Method & FID $\downarrow$ & FID$_\text{p}\downarrow$ & KID $\downarrow$ & KID$_\text{p}\downarrow$ & CLIP Score $\uparrow$ & Latency(s) $\downarrow$ \\
    \hline
     $N = 3$ &65.87 &52.88 &0.0090 &0.0082 &31.07 &114 \\
     $N = 4$ &65.09 &52.26 &0.0083 &0.0076 &31.11 &115 \\
     $N = 5$ &64.91 &51.84 &0.0078 &0.0074 &31.17 &116 \\
     $N = 6$ &64.92 &51.82 &0.0081 &0.0073 &31.15 &117 \\
     $N = 7$ &65.04 &52.29 &0.0082 &0.0075 &31.09 &117 \\
    \hline
  \end{tabular}
  }
  \label{table: ablation_diff_bundle_size}
\end{table}

\noindent\textbf{Impact of the ratio $r_c$.} In our framework, we set $r_c$, denoting the target ratio of total computational cost for all multi-head attention operations measured relative to that under full attention, to be 0.1. Here to evaluate the impact of the ratio $r_c$, we evaluate different choices of $r_c$ ranging from 0.01 to 0.5, and report the results in \cref{table: ablation_diff_rc}. As shown, when $r_c$ exceeds 0.1, the generation quality improvement is trivial. Thus, taking the generation efficiency also into consideration, we set $r_c$ to 0.1 in our experiments.

\begin{table}[h]
  \caption{Evaluation on the different $r_c$.}
  \centering
  \resizebox{0.6\linewidth}{!}{
  \begin{tabular}{lcccccc}
    \toprule
    Method & FID $\downarrow$ & FID$_\text{p}\downarrow$ & KID $\downarrow$ & KID$_\text{p}\downarrow$ & CLIP Score $\uparrow$ & Latency(s) $\downarrow$ \\ 
    \hline
     $r_c = 0.01$ &94.51 &91.34 &0.0383 &0.0379 &27.39 &107 \\ 
     $r_c = 0.05$ &75.66 &67.49 &0.0197 &0.0181 &30.07 &111 \\ 
     $r_c = 0.1$ &64.91 &51.84 &0.0078 &0.0074 &31.17 &116 \\ 
     $r_c = 0.3$ &64.90 &51.84 &0.0078 &0.0074 &31.18 &134 \\ 
     $r_c = 0.5$ &64.88 &51.83 &0.0077 &0.0074 &31.18 &155 \\
     
    \hline
  \end{tabular}
  }
  \label{table: ablation_diff_rc}
\end{table}

\noindent\textbf{Impact of the number of input prompts used in preparatory step 1.} In preparatory step 1 of our framework, to stabilize the estimation, we average $I_q$ over 30 input prompts. Here, we also test other choices of the number of input prompts used in preparatory step 1 and report the results in \cref{table:ablation_num_prompts}. As shown, the generation quality improves when we increase the number of input prompts used in preparatory step 1 from 10 to 30, and the improvement of generation quality becomes minimal when the number is further increased. Hence, taking the time taken during the preparatory stage also into consideration, we average $I_q$ over 30 input prompts during preparatory step 1.

\begin{table}[h]
  \caption{Evaluation on the number of input prompts used in preparatory step 1.}
  \centering
  \resizebox{0.8\linewidth}{!}{
  \begin{tabular}{lccccccc}
    \hline
    Method & FID $\downarrow$ & FID$_\text{p}\downarrow$ & KID $\downarrow$ & KID$_\text{p}\downarrow$ & CLIP Score $\uparrow$ & Latency(s) $\downarrow$ & Preparation Time(min) $\downarrow$ \\
    \hline
    10 input prompts &66.13 &54.68 &0.0092 &0.0081 &31.02 &116 &18 \\
    20 input prompts &65.02 &52.15 &0.0081 &0.0077 &31.14 &116 &33 \\
    30 input prompts &64.91 &51.84 &0.0078 &0.0074 &31.17 &116 &48 \\
    40 input prompts &64.91 &51.80 &0.0077 &0.0072 &31.17 &116 &63 \\
    50 input prompts &64.90 &51.79 &0.0077 &0.0071 &31.17 &116 &78 \\
    \hline
  \end{tabular}
  }
  \label{table:ablation_num_prompts}
\end{table}

\noindent\textbf{Impact of performing the one-time preparatory stage separately for each testing resolution.} In our framework (\textbf{preparation once per testing resolution}), we perform the one-time preparatory stage separately for each testing resolution. For comparison, we also evaluate a variant (\textbf{preparation once for all testing resolutions}), in which the preparatory stage is executed only once at the 1K resolution and the resulting collective scope assignment is then reused for all testing resolutions. As shown in Tab.~\ref{table: diff_res_validset}, this variant also outperforms the state-of-the-art method HiFlow~\cite{hiflow}.

\begin{table}[h]
  \caption{Evaluation on performing the one-time preparatory stage separately for each testing resolution.}
  \centering
  \resizebox{0.7\linewidth}{!}{
  \begin{tabular}{lcccccc}
    \toprule
    Method & FID $\downarrow$ & FID$_\text{p}\downarrow$ & KID $\downarrow$ & KID$_\text{p}\downarrow$ & CLIP Score $\uparrow$ & Latency(s) $\downarrow$ \\
    \hline
    HiFlow~\cite{hiflow} &68.36 &57.43 &0.0124 &0.0114 &30.62 &209 \\
    \hline
    Preparation once for all testing resolutions &66.09 &53.11 &0.0087 &0.0079 &31.04 &116 \\
    Preparation once per testing resolution &64.91 &51.84 &0.0078 &0.0074 &31.17 &116 \\
    \bottomrule
  \end{tabular}
  }
  \label{table: diff_res_validset}
\end{table}

\noindent\textbf{Impact of determining the collective scope assignment plan from the final denoising step.} In our framework (\textbf{single plan determined using only the final denoising step}), when determining the collective scope assignment for a given testing resolution, we estimate $I_q$ using only the output from the final denoising step at that resolution. We also evaluate two variants. In the first variant (\textbf{single plan determined using all denoising steps}), we compute $I_q$ at each denoising step separately and then average these values to obtain the final $I_q$, rather than relying solely on the final step. In the second variant (\textbf{separate plan per denoising step}), we determine an independent collective scope assignment for each denoising step, using the $I_q$ estimated at that step. As shown in Tab.~\ref{table: scope_final}, both variants achieve performance comparable to our framework. This suggests that the inherent attention scopes of different heads in the multi-head attention of off-the-shelf DiT models, when extended to high resolutions, tend to remain stable across denoising steps. For simplicity, we therefore determine the collective scope assignment for each testing resolution using only the final denoising step at that resolution.

\begin{table}[h]
  \caption{Evaluation on determining the collective scope assignment plan from the final denoising step.}
  \centering
  \resizebox{0.8\linewidth}{!}{
  \begin{tabular}{lcccccc}
    \toprule
    Method & FID $\downarrow$ & FID$_\text{p}\downarrow$ & KID $\downarrow$ & KID$_\text{p}\downarrow$ & CLIP Score $\uparrow$ & Latency(s) $\downarrow$ \\
    \hline
    HiFlow~\cite{hiflow} &68.36 &57.43 &0.0124 &0.0114 &30.62 &209 \\
    \hline
    Single plan determined using all denoising steps &64.91 &51.84 &0.0077 &0.0073 &31.17 &116 \\
    Separate plan per denoising step &64.90 &51.84 &0.0077 &0.0073 &31.18 &117 \\
    Single plan determined using only the final denoising step &64.91 &51.84 &0.0078 &0.0074 &31.17 &116 \\
    \bottomrule
  \end{tabular}
  }
  \label{table: scope_final}
\end{table}

\noindent\textbf{Impact of using text prompts randomly sampled from LAION-5B during estimating $I_q$.} In our framework (\textbf{Prompts from LAION-5B during estimating $I_q$}), during the estimation of $I_q$, we stabilize the process by averaging $I_q$ over 30 text prompts randomly sampled from the LAION-5B dataset~\cite{schuhmann2022laion}, while explicitly excluding any prompts used during evaluation. Here, we test two variants in which, during the estimation of $I_q$, we use prompts generated by ChatGPT instead of sampling them from LAION-5B. In the first variant (\textbf{Prompt generated by ChatGPT during estimating $I_q$ (first generation strategy)}), we simply ask ChatGPT to generate 30 text prompts without referencing the prompts used during evaluation. In the second variant (\textbf{Prompt generated by ChatGPT during estimating $I_q$ (second generation strategy)}), we again generate 30 prompts using ChatGPT, but this time we provide ChatGPT with all 1,000 evaluation prompts and explicitly instruct it to produce 30 diverse prompts that avoid generating any that are similar to these evaluation prompts. As shown in Tab.~\ref{table:prompts_choose}, both variants achieve performance comparable to our framework. This demonstrates that the estimation of $I_q$ is robust to the source of the text prompts used.

\begin{table}[h]
  \caption{Evaluation on using text prompts randomly sampled from LAION-5B during estimating $I_q$.}
  \centering
  \resizebox{0.9\linewidth}{!}{
  \begin{tabular}{lcccccc}
    \toprule
    Method & FID $\downarrow$ & FID$_\text{p}\downarrow$ & KID $\downarrow$ & KID$_\text{p}\downarrow$ & CLIP Score $\uparrow$ & Latency(s) $\downarrow$ \\
    \hline
    Prompt generated by ChatGPT during estimating $I_q$ (first generation strategy) &64.92 &51.84 &0.0079 &0.0075 &31.17 &116 \\
    Prompt generated by ChatGPT during estimating $I_q$ (second generation strategy) &64.93 &51.82 &0.0080 &0.0074 &31.17 &116 \\
    Prompts from LAION-5B during estimating $I_q$ &64.91 &51.84 &0.0078 &0.0074 &31.17 &116 \\
    \bottomrule
  \end{tabular}
  }
  \label{table:prompts_choose}
\end{table}

\section{Additional Experiments}

\noindent\textbf{Comparison with image super-resolution method.} In the community, besides directly performing text-to-high-resolution generation, another possible approach for high-resolution image synthesis is to first generate limited-resolution images using an off-the-shelf DiT model and then upscale them using an image super-resolution method. For completeness, we also include this approach in our comparisons. However, it is important to note that image super-resolution methods are typically training-based: obtaining a high-performing super-resolution model generally requires large high-resolution training datasets and substantial computational resources. In contrast, our framework operates entirely in a training-free manner. On the 4K image generation task, following \cite{zhang2025frecas}, we compare our method with the recent powerful super-resolution model SUPIR \cite{yu2024scaling}. As shown in \cref{table:super_res}, our method consistently achieves both higher generation quality and faster generation speed than SUPIR across both Stable Diffusion 3 (SD3) and FLUX, even though SUPIR requires high-resolution training while our framework is training-free. This further demonstrates the effectiveness of our approach.

\begin{table}[h]
  \caption{Comparison between the image super-resolution method SUPIR and our approach. Notably, though SUPIR requires high-resolution training, our training-free approach still outperforms it, further demonstrating the effectiveness of our approach.}
  \centering
  \resizebox{0.8\linewidth}{!}{
  \begin{tabular}{c|lcccccc}
    \hline
    Resolution &
    Method & FID $\downarrow$ & FID$_\text{p}\downarrow$ & KID $\downarrow$ & KID$_\text{p}\downarrow$ & CLIP Score $\uparrow$ & Latency(s) $\downarrow$ \\
    \hline 
    \multirow{4}{*}{$\mathbf{2k}$} &SD3~\cite{sd3} + SUPIR~\cite{yu2024scaling} &68.49 &53.64 &0.0088 &0.0077 &31.75 &59 \\
    &SD3~\cite{sd3} + Ours &\textbf{67.90} &\textbf{52.59} &\textbf{0.0081} &\textbf{0.0073} &\textbf{31.81} &\textbf{11} \\
    \cline{2-8}
    & FLUX~\cite{blackforestlabs2024flux} + SUPIR~\cite{yu2024scaling} &65.25 &53.67 &0.0083 &0.0072 &31.22 &59 \\
    & FLUX~\cite{blackforestlabs2024flux} + Ours &\textbf{64.42} &\textbf{51.55} &\textbf{0.0074} &\textbf{0.0067} &\textbf{31.27} &\textbf{35} \\
    \hline
    \multirow{4}{*}{$\mathbf{4k}$} &SD3~\cite{sd3} + SUPIR~\cite{yu2024scaling} &69.33 &54.57 &0.0097 &0.0084 &31.70 &305 \\
    &SD3~\cite{sd3} + Ours &\textbf{68.63} &\textbf{53.04} &\textbf{0.0089} &\textbf{0.0079} &\textbf{31.79} &\textbf{58} \\
    \cline{2-8}
    & FLUX~\cite{blackforestlabs2024flux} + SUPIR~\cite{yu2024scaling} &65.89 &54.96 &0.0092 &0.0083 &31.04 &305 \\
    & FLUX~\cite{blackforestlabs2024flux} + Ours &\textbf{64.91} &\textbf{51.84} &\textbf{0.0078} &\textbf{0.0074} &\textbf{31.17} &\textbf{116} \\
    \hline
  \end{tabular}
  }
  \label{table:super_res}
\end{table}

\section{Additional Details about Eq. 1 in the Main Paper}
\label{sec:appendix_a}

Denote by $\mathbf{x}_i \in \mathbb{R}^{1 \times d_{\text{input}}}$ the $i$-th token in the input hidden state $\mathbf{x} \in \mathbb{R}^{T \times d_{\text{input}}}$ of an attention block, where $T = h_{\text{hidden}} w_{\text{hidden}}$, $h_{\text{hidden}} = h_{\text{image}} / 16$, and $w_{\text{hidden}} = w_{\text{image}} / 16$. Here, $h_{\text{image}}$ and $w_{\text{image}}$ respectively denote the height and width of the input image to the DiT model, and we assume both dimensions are multiples of 16, a typical requirement for DiT models. (For inputs that do not satisfy this condition, the image is typically resized or padded before being fed into the model). Let $\mathbf{q_i} \in \mathbb{R}^{1 \times d}$, $\mathbf{k_i} \in \mathbb{R}^{1 \times d}$, and $\mathbf{v_i} \in \mathbb{R}^{1 \times d}$ be the corresponding query, key, and value vectors derived from $\mathbf{x_i}$. Recall that in Eq.~1 of the main paper, we point out that, for mainstream off-the-shelf DiT models including FLUX~\cite{blackforestlabs2024flux} and Stable Diffusion~3~\cite{sd3}, their procedures for deriving $\mathbf{q_i}$, $\mathbf{k_i}$, and $\mathbf{v_i}$ from $\mathbf{x_i}$, though varying in detailed designs, can be unified using the following formulation through appropriately defined functions $g_{q}(\cdot,\cdot)$, $g_{k}(\cdot,\cdot)$, $g_{v}(\cdot,\cdot)$, and $f^{tok}_{pe}(\cdot)$:
\begin{equation}
\setlength{\abovedisplayskip}{3pt}
\setlength{\belowdisplayskip}{3pt}
\begin{aligned}
\label{Eq:1_restate}
\mathbf{q_i}=g_{q}(\mathbf{x_i},f^{tok}_{pe}(i)), \mathbf{k_i}=g_{k}(\mathbf{x_i},f^{tok}_{pe}(i)), \mathbf{v_i}=g_{v}(\mathbf{x_i},f^{tok}_{pe}(i)),
\end{aligned}
\end{equation}

In this section, our goal is to substantiate this claim by showing that the computational processes of both models can indeed be written in the form of Eq.~\ref{Eq:1_restate}. To this end, for each model, we explicitly construct the corresponding definitions of the four functions $g_{q}(\cdot,\cdot)$, $g_{k}(\cdot,\cdot)$, $g_{v}(\cdot,\cdot)$, and $f^{tok}_{pe}(\cdot)$ such that, once instantiated, Eq.~\ref{Eq:1_restate} precisely reproduces that model’s procedure for computing $\mathbf{q_i}$, $\mathbf{k_i}$, and $\mathbf{v_i}$ from $\mathbf{x_i}$.

\noindent\textbf{Instantiation for FLUX.} FLUX uses rotary position embeddings (RoPE)~\cite{pe1}, applied to both queries and keys. When applying RoPE, FLUX first maps the token index $i$ to its corresponding row and column indices, $i_r = i // w_{\text{hidden}}$ and $i_c = i \% w_{\text{hidden}}$. For $\mathbf{q}_i \in \mathbb{R}^{1 \times d}$ and $\mathbf{k}_i \in \mathbb{R}^{1 \times d}$, FLUX further decomposes the dimensionality as $d = d_{\text{inf}} + d_{\text{row}} + d_{\text{col}}$, where $d_{\text{inf}}$ is the dimension associated with information that does not depend on row or column indices, $d_{\text{row}}$ encodes the row index $i_r$, and $d_{\text{col}}$ encodes the column index $i_c$. By design, $d_{\text{inf}}$, $d_{\text{row}}$, $d_{\text{col}}$, and $d$ are all even numbers. With these notations, for each token index~$i$, the token-level positional signal $f^{tok}_{pe}(\cdot)$ of FLUX can be instantiated as:

\begin{equation}
\begin{aligned}
f^{tok}_{pe}(i) &= \{\theta_m(i)\}_{m=1}^{d/2}, \\
\text{where}\quad
\theta_m(i) &=
\begin{cases}
0, & 1 \le m \le \dfrac{d_{\text{inf}}}{2}, \\[6pt]
i_r \cdot \omega_m,
& \dfrac{d_{\text{inf}}}{2} < m \le \dfrac{d_{\text{inf}} + d_{\text{row}}}{2}, \\[6pt]
i_c \cdot \omega_m,
& \dfrac{d_{\text{inf}} + d_{\text{row}}}{2}
< m \le
\dfrac{d_{\text{inf}} + d_{\text{row}} + d_{\text{col}}}{2}
= \dfrac{d}{2}.
\end{cases}
\end{aligned}
\end{equation}
where 
$\omega_m = 10000^{-2(m-1)/d}$. Using this definition, the functions $g_{q}(\cdot,\cdot)$, $g_{k}(\cdot,\cdot)$, $g_{v}(\cdot,\cdot)$ for FLUX can be written as:
\begin{equation}
\begin{aligned}
g_q(\mathbf{x_i}, f^{tok}_{pe}(i))
&= \mathrm{RoPE}_{d}(\mathbf{x_i} W_q^{flux},\; f^{tok}_{pe}(i)), \\
g_k(\mathbf{x_i}, f^{tok}_{pe}(i))
&= \mathrm{RoPE}_{d}(\mathbf{x_i} W_k^{flux},\; f^{tok}_{pe}(i)), \\
g_v(\mathbf{x_i}, f^{tok}_{pe}(i))
&= \mathbf{x_i} W_v^{flux},
\end{aligned}
\label{Eq:1_flux}
\end{equation}
where $W_q^{flux}, W_k^{flux}, W_v^{flux} \in \mathbb{R}^{d_{\text{input}} \times d}$ are the learned linear projection matrices for queries, keys, and values, respectively.
Below, we elaborate on the operation of $\mathrm{RoPE}_{d}(\cdot, \cdot)$. Let $z = [z_1, z_2, \dots, z_{d/2}] \in \mathbb{R}^{1 \times d}$ denote the projected query or key vector (i.e., $z = \mathbf{x_i} W_q^{flux}$ or $z = \mathbf{x_i} W_k^{flux}$), where each $z_m \in \mathbb{R}^{1 \times 2}$. RoPE applies a rotation to each $z_m$ according to its positional angle $\theta_m(i)$: 
\begin{equation}
\begin{aligned}
\mathrm{RoPE}_{d}(z,f^{tok}_{pe}(i))
=
&\big[
(R(\theta_1(i))(z_1)^\intercal)^\intercal,\;
(R(\theta_2(i))(z_2)^\intercal)^\intercal, \\
&\quad
\dots,\;
(R(\theta_{d/2}(i))(z_{d/2})^\intercal)^\intercal
\big].
\end{aligned}
\end{equation}
where $R(\theta)=
\begin{bmatrix}
\cos\theta & -\sin\theta \\
\sin\theta & \cos\theta
\end{bmatrix}
$ is the standard $2\times2$ rotation matrix used in RoPE~\cite{pe1}.

The above shows that, for FLUX, we can explicitly instantiate $g_{q}(\cdot,\cdot)$, $g_{k}(\cdot,\cdot)$, $g_{v}(\cdot,\cdot)$, and $f^{tok}_{pe}(\cdot)$ such that, once instantiated, Eq.~\ref{Eq:1_restate} exactly reproduces FLUX's procedure for computing $\mathbf{q_i}$, $\mathbf{k_i}$, and $\mathbf{v_i}$ from $\mathbf{x_i}$. This demonstrates that FLUX fits precisely into the unified formulation of Eq.~\ref{Eq:1_restate}.

\noindent\textbf{Instantiation for Stable Diffusion 3.} Stable Diffusion~3 employs sinusoidal positional embeddings~\cite{pe2}, applied to queries, keys, and values. Notably, these positional embeddings are used only in the first attention block. In this block, leveraging the sinusoidal embeddings, the four functions for Stable Diffusion~3~\cite{sd3} are instantiated as:
\begin{equation}
\begin{aligned}
f^{tok}_{pe}(i) &= \mathbf{p_i}, \\
g_q(\mathbf{x_i}, f^{tok}_{pe}(i)) &= (\mathbf{x_i} + f^{tok}_{pe}(i)) W_q^{sd3}, \\
g_k(\mathbf{x_i}, f^{tok}_{pe}(i)) &= (\mathbf{x_i} + f^{tok}_{pe}(i)) W_k^{sd3}, \\
g_v(\mathbf{x_i}, f^{tok}_{pe}(i)) &= (\mathbf{x_i} + f^{tok}_{pe}(i)) W_v^{sd3},
\end{aligned}
\label{Eq:sd3}
\end{equation}
where $W_q^{sd3}, W_k^{sd3}, W_v^{sd3}$ are the learned linear projection matrices for queries, keys, and values, respectively. Meanwhile, for $d_{\text{input}}$ by design an even number, we have $\mathbf{p_i} = [\sin(\frac{i}{10000^{0/d_{\text{input}}}}), \cos(\frac{i}{10000^{0/d_{\text{input}}}}), \sin(\frac{i}{10000^{2/d_{\text{input}}}}), \cos(\frac{i}{10000^{2/d_{\text{input}}}}),$
$\dots, \sin(\frac{i}{10000^{(d_{\text{input}} - 2)/d_{\text{input}}}}), \cos(\frac{i}{10000^{(d_{\text{input}} - 2)/d_{\text{input}}}})] \in \mathbb{R}^{1 \times d_{\text{input}}}$. The above shows that, for Stable Diffusion 3, we can also explicitly instantiate its corresponding $g_{q}(\cdot,\cdot)$, $g_{k}(\cdot,\cdot)$, $g_{v}(\cdot,\cdot)$, and $f^{tok}_{pe}(\cdot)$, making Stable Diffusion 3 also fits precisely into the unified formulation of Eq.~\ref{Eq:1_restate}.

Taken together, the above analysis indicates that both FLUX and Stable Diffusion 3 conform exactly to the unified formulation specified in Eq.~1 in the main paper.

\section{Additional Details about Eq. 3 in the Main Paper}
\label{sec:appendix_b}

In Eq.~3 of the main paper, we further point out that for mainstream off-the-shelf DiT models, including FLUX~\cite{blackforestlabs2024flux} and Stable Diffusion~3~\cite{sd3}, by combining Eq.~1 and Eq.~2 in the main paper, the term $\mathbf{c}_{i,j}$ (the contribution of token $j$ to the attention output at position $i$) can be expressed in a unified form using a single function $g$:
\begin{equation}
\setlength{\abovedisplayskip}{3pt}
\setlength{\belowdisplayskip}{3pt}
\label{Eq:3_supp}
\mathbf{c}_{i,j} = g(\mathbf{x_i}, \mathbf{x_j}, f_{pe}(i, j)),
\end{equation}
where $f_{pe}(i, j)$ denotes the pairwise positional signal.

In this section, following the procedure in Sec.~\ref{sec:appendix_a}, we substantiate this claim by explicitly constructing the definitions of $f_{pe}(\cdot,\cdot)$ and $g(\cdot,\cdot,\cdot)$ for both FLUX and Stable Diffusion~3. Specifically, given (1) the instantiation of Eq.~1 in the main paper (as detailed in Sec.~\ref{sec:appendix_a}) and (2) Eq.~2 in the main paper, we show that by combining them, the corresponding forms of $f_{pe}(\cdot,\cdot)$ and $g(\cdot,\cdot,\cdot)$ for FLUX and Stable Diffusion~3 can be explicitly constructed as follows.

\noindent\textbf{Instantiation for FLUX.} For FLUX~\cite{blackforestlabs2024flux}, we follow the notations in Sec.~\ref{sec:appendix_a}. We additionally denote $z^{q}_{i} = \mathbf{x_i}W_q^{flux} =  [z^{q}_{i, 1}, z^{q}_{i, 2}, \dots, z^{q}_{i, d/2}] \in \mathbb{R}^{1 \times d}$ where $z^{q}_{i, m} \in \mathbb{R}^{1\times2}$, and denote $z^{k}_{j} = \mathbf{x_j}W_k^{flux} = [z^{k}_{j, 1}, z^{k}_{j, 2}, \dots, z^{k}_{j, d/2}] \in \mathbb{R}^{1 \times d}$ where $z^{k}_{j, m} \in \mathbb{R}^{1\times2}$. Moreover, we derive $j_r$ and $j_c$ from the token index $j$ in the same way as we derive $i_r$ and $i_c$ from the token index $i$. $f_{pe}(\cdot,\cdot)$ and $g(\cdot,\cdot,\cdot)$ of FLUX can then be instantiated as:

\begin{equation}
f_{pe}(i,j) = \{f_{pe}(i,j)[m]\}_{m=1}^{d/2}
\end{equation}

\begin{equation}
f_{pe}(i,j)[m] =
\begin{cases}
0, & 1 \le m \le \dfrac{d_{\text{inf}}}{2}, \\[6pt]
(i_r-j_r)\,\omega_m,
& \dfrac{d_{\text{inf}}}{2} < m \le
\dfrac{d_{\text{inf}} + d_{\text{row}}}{2}, \\[6pt]
(i_c-j_c)\,\omega_m,
& \dfrac{d_{\text{inf}} + d_{\text{row}}}{2}
< m \le
\dfrac{d_{\text{inf}} + d_{\text{row}} + d_{\text{col}}}{2}
= \dfrac{d}{2}.
\end{cases}
\end{equation}

\begin{equation}
\begin{aligned}
g(\mathbf{x_i},\mathbf{x_j},f_{pe}(i,j))
&=
\frac{
\exp\!\Big(
\sum_{m=1}^{d/2}
z^{q}_{i,m}\,
R\!\big(-f_{pe}(i,j)[m]\big)\,
(z^{k}_{j,m})^\intercal
\Big)
}
{
\sum_{t=1}^{T}
\exp\!\Big(
\sum_{m=1}^{d/2}
z^{q}_{i,m}\,
R\!\big(-f_{pe}(i,t)[m]\big)\,
(z^{k}_{t,m}\big)^\intercal
\Big)
}
\, W_v^{flux}\mathbf{x_j}.
\end{aligned}
\label{Eq:3_supp_flux}
\end{equation}

\noindent\textbf{Instantiation for Stable Diffusion 3.} Meanwhile, for Stable Diffusion~3~\cite{sd3}, we follow the notations in Sec.~\ref{sec:appendix_a} and additionally denote $A=(W_q^{sd3})^\intercal W_k^{sd3} \in \mathbb{R}^{d \times d}$. In the first attention block of Stable Diffusion~3 where positional embeddings are applied, $f_{pe}(\cdot,\cdot)$ and $g(\cdot,\cdot,\cdot)$ can then be instantiated as:

\begin{equation}
\begin{aligned}
f_{pe}(i,j) &= \{\mathbf{p_i},\mathbf{p_j}\}, \\[4pt]
s(i,j)
&=
\mathbf{x_i}^\intercal A\,\mathbf{x_j}
+
\mathbf{x_i}^\intercal A\, (f_{pe}(i,j)[1]) \\
&\quad
+
(f_{pe}(i,j)[0])^\intercal A\,\mathbf{x_j}
+
(f_{pe}(i,j)[0])^\intercal A\,(f_{pe}(i,j)[1]).
\end{aligned}
\end{equation}

\begin{equation}
g(\mathbf{x_i},\mathbf{x_j},f_{pe}(i,j))
=
\frac{\exp(s(i,j))}
{\sum_{t=1}^{T}\exp(s(i,t))}
\, W_v^{sd3}(\mathbf{x_j}+f_{pe}(i,j)[1]).
\label{Eq:3_supp_sd3}
\end{equation}

Taken together, Eq.~\ref{Eq:3_supp_flux} and Eq.~\ref{Eq:3_supp_sd3} show that, both FLUX and Stable Diffusion 3 conform exactly to the unified formulation specified in Eq.~3 in the main paper. 

\section{Additional Details about $|S_{pe}|$}

Based on the above instantiation of $f_{pe}(\cdot,\cdot)$ and $g(\cdot,\cdot,\cdot)$ respectively for FLUX and Stable Diffusion 3, we can derive the size of $S_{pe}$ (i.e., $|S_{pe}|$), for FLUX and Stable Diffusion 3 respectively. Specifically, recall that $|S_{pe}|$ denotes the size of $S_{pe} = \{\, f_{pe}(i,j) \mid i \in \{0,\dots,T-1\},~ j \in \{0,\dots,T-1\} \,\}$. For FLUX, following the general merits of RoPE for which $f_{pe}$ has $2T-1$ distinct outputs when $i,j \in \{0,\dots,T-1\}$, also accounting for the 2-dimensional nature of Eq.~\ref{Eq:3_supp_flux} (considering both $i_r$ and $i_c$), recall that in our case $T = h_{\text{hidden}} w_{\text{hidden}}$. We can have $f_{pe}$ attains $(2h_{\text{hidden}} -1)(2w_{\text{hidden}} -1)$ distinct outputs for FLUX. Concretely, when $h_{\text{hidden}} = w_{\text{hidden}} = \sqrt{T}$, this leads $|S_{pe}| = (2\sqrt{T} -1)^2$ for FLUX.  For Stable Diffusion~3, from that $f_{pe}(i,j) = \{\mathbf{p_i},\mathbf{p_j}\}$, we find that every distinct pair $(i,j)$ yields a distinct output of $f_{pe}$; therefore, $f_{pe}$ attains $T^2$ distinct outputs when $i,j \in \{0,\dots,T-1\}$. We can accordingly have $|S_{pe}| = T^2$ for Stable Diffusion~3. The above shows that, as $T$ grows rapidly, $|S_{pe}|$ typically scales accordingly for both FLUX and Stable Diffusion 3.

\section{More Details of Proposition 1 in the Main Paper}

In this section, we provide additional details on the adapted proof of Proposition 1 in the main paper. The proof follows a similar argument to that in \cite{han2023lm}.

\setcounter{prop}{0}%
\begin{prop}
\label{prop:1_restate}
Let $dis_{g}$ denote a distance measure quantifying the discrepancy between two pairwise positional signals in the output space of the function $g$.
By definition of $g$ in Eq.~3 in the main paper, $dis_{g}$ then measures the discrepancy between two pairwise positional signals, after they are respectively propagated through $g$ during attention computation.
Using $dis_{g}$, we partition all elements in the set $S_{pe}$ into $h(T)$ mutually exclusive subsets such that any two elements within the same subset satisfy $dis_{g} \le \epsilon$, while any two elements from different subsets satisfy $dis_{g} > \epsilon$, where $\epsilon$ is a predefined threshold.
Intuitively, the resulting $h(T)$ characterizes the positional embedding mechanism’s ability to distinguish between different pairwise positional signals in $S_{pe}$ after attention computation.
A larger $h(T)$, indicating that $S_{pe}$ can be divided into a greater number of well-separated subsets under $dis_{g}$, corresponds to a stronger distinction capability of the positional embedding mechanism.
Yet meanwhile, $h(T)$ can be proved to be upper-bounded by:
\begin{equation}
\setlength{\abovedisplayskip}{3pt}
\setlength{\belowdisplayskip}{3pt}
\label{eq:3_restate}
h(T) \;\le\; \lambda \cdot \Big(\sup_{S_{pe}} |g(\mathbf{x_i}, \mathbf{x_j}, f_{pe}(i,j))|\Big)^{2\xi},
\end{equation}
where $\xi$ denotes the pseudo-dimension of the function class $\mathcal{G} = \{\, g(\cdot, \cdot, f_{pe}(i,j)) \mid i \in \{0, \dots, T-1\},~ j \in \{0, \dots, T-1\} \,\}$, and $\lambda = (\frac{e}{\epsilon})^{2\xi} \cdot 2^{4\xi+1}$. 
Notably, $\lambda$ can be regarded as a constant when the threshold $\epsilon$ is fixed.
\end{prop}

Before presenting the proof of Proposition~\ref{prop:1_restate}, the distance measure $dis_g$ is first formally defined. Specifically, for $dis_{g}$ to be able to quantify the discrepancy between two pairwise positional signals in the output space of the function $g$, $dis_{g}$ is defined with the help of expectation as:
\begin{equation}
\setlength{\abovedisplayskip}{3pt}
\setlength{\belowdisplayskip}{3pt}
\label{eq:dis_g}
\begin{aligned}
dis_{g}(f_{pe}(i, j), f_{pe}(i', j'))
= \mathbb{E}_{\mathbf{x_i}\sim\mathbf{X_i},\, \mathbf{x_j}\sim\mathbf{X_j}}
\Big[
\Big(
g(\mathbf{x_i}, \mathbf{x_j}, f_{pe}(i, j)) \\
\quad
- g(\mathbf{x_i}, \mathbf{x_j}, f_{pe}(i', j'))
\Big)^2
\Big]
\end{aligned}
\end{equation}
where $\mathbf{X_i}$ and $\mathbf{X_j}$ respectively denote the trained distributions of $\mathbf{x_i}$ and $\mathbf{x_j}$.

After $dis_g$ is defined, the proof of Proposition~\ref{prop:1_restate} is presented. To this end, the following lemma, drawn on from \cite{haussler1992decision,han2023lm}, is first introduced to facilitate the proof of the proposition.

\begin{lemma}
\label{lemma}
Let $\mathcal{J}=\{j(u)\}$ be a collection of functions defined on a set $\mathcal{U}$, each taking values in $[0,\beta]$.  
Assume that the pseudo-dimension of $\mathcal{J}$ is $dim_{\text{pseudo}}(\mathcal{J}) = \xi$ for some finite $\xi$. Let $\mathbf{Q}$ denote a probability measure on $\mathcal{U}$.  
For any $0 < \epsilon \le \beta$, consider the squared $L_2(\mathbf{Q})$ metric:
\begin{equation}
\rho(j_1,j_2)
    = \mathbb{E}_{u \sim \mathbf{Q}}
      \bigl[(j_1(u)-j_2(u))^2\bigr].
\end{equation}
Then the $\epsilon$-cover of $\mathcal{J}$ with respect to $\rho$ satisfies:
\begin{equation}
\mathcal{C}_{\mathbf{Q}}(\epsilon,\mathcal{J},\rho)
    \;\le\;
    2\left(
        \frac{2e\beta}{\epsilon}
        \,\ln\!\frac{2e\beta}{\epsilon}
      \right)^{\xi}.
\end{equation}
Here, $\mathcal{C}_{\mathbf{Q}}(\epsilon,\mathcal{J},\rho)$ denotes the smallest size of a cover-set $\mathcal{J}_{0}$ such that for every $j\in\mathcal{J}$, there exists $j_{0}\in\mathcal{J}_{0}$ with $\rho(j,j_{0}) \le \epsilon$.
\end{lemma}

With this lemma introduced, Proposition~\ref{prop:1_restate} can then be proved as follows, following the proof strategy of \cite{han2023lm}.

\begin{proof}
Assume, for the sake of contradiction, that the statement in Eq.~\ref{eq:3_restate} of Proposition~\ref{prop:1_restate} fails to hold.  
That is, suppose
\begin{equation}
\label{eq:proof_1_1}
h(T) > \lambda \cdot \Big(\sup_{S_{pe}} |g(\mathbf{x_i}, \mathbf{x_j}, f_{pe}(i,j))|\Big)^{2\xi},
\end{equation}
Via rewriting Eq.~\ref{eq:proof_1_1}, since both $\xi$ and $\lambda$ are positive, the following can be obtained:
\begin{equation}
\label{eq:proof_1_2}
\sup_{S_{pe}} |g(\mathbf{x_i}, \mathbf{x_j}, f_{pe}(i,j))| < (\frac{h(T)}{\lambda})^{\frac{1}{2\xi}}
\end{equation}
To facilitate the following proof, denote $\tau := (\frac{h(T)}{\lambda})^{\frac{1}{2\xi}}$. Note that, Eq.~\ref{eq:proof_1_2} means that, each function in $\mathcal{G} = \{\, g(\cdot, \cdot, f_{pe}(i,j)) \mid i \in \{0, \dots, T-1\},~ j \in \{0, \dots, T-1\} \,\}$ takes values strictly inside the interval $[-\tau,\tau]$. Without loss of generality, the interval $[-\tau,\tau]$ can be shifted to $[0,2\tau]$. Under the mild assumption that the family
$\mathcal{G} = \{, g(\cdot, \cdot, f_{pe}(i,j)) \mid i \in {0, \dots, T-1},~ j \in {0, \dots, T-1}\}$
has finite pseudo-dimension, Lemma~\ref{lemma} can then be directly applied. This assumption is mild in our setting and is consistent with prior analyses of commonly used model families \cite{han2023lm,li2025longdiff}.

From Lemma~\ref{lemma}, the $\epsilon$-cover of $\mathcal{G}$ with respect to $\rho$ satisfies:
\begin{equation}
\label{eq:proof_1_3}
\mathcal{C}_{\mathbf{Q}}(\epsilon,\mathcal{G},dis_g)
    \;\le\;
    2\left(
        \frac{2e\times 2\tau}{\epsilon}
        \,\ln\!\frac{2e\times 2\tau}{\epsilon}
      \right)^{\xi}.
\end{equation}
Recall that $\tau = (\frac{h(T)}{\lambda})^{\frac{1}{2\xi}}$ and  $\lambda = (\frac{e}{\epsilon})^{2\xi} \cdot 2^{4\xi+1}$. Using these, Eq.~\ref{eq:proof_1_3} can be rewritten into:
\begin{equation}
\mathcal{C}_{\mathbf{Q}}(\epsilon,\mathcal{G},dis_g)
    \le 
    2 \left(
      \left(\frac{h(T)}{2}\right)^{\!\frac{1}{2\xi}}
      \ln\!\left(\frac{h(T)}{2}\right)^{\!\frac{1}{2\xi}}
    \right)^{\xi}.
\end{equation}
Since $\ln x < x$ for all $x>0$, the following can be further obtained:
\begin{equation}
\mathcal{C}_{\mathbf{Q}}(\epsilon,\mathcal{G},dis_g)
    < 
    2 \left(
      \left(\frac{h(T)}{2}\right)^{\!\frac{1}{2\xi}}
      \left(\frac{h(T)}{2}\right)^{\!\frac{1}{2\xi}}
    \right)^{\xi} = h(T).
\end{equation}

The above means that, under the assumption in Eq.~\ref{eq:proof_1_1}, the $\epsilon$-cover $\mathcal{C}_{\mathbf{Q}}(\epsilon,\mathcal{G},dis_g)$ cannot reach $h(T)$. Instead, it needs to in fact be strictly smaller. Equivalently, within the function family $\mathcal{G} = \{\, g(\cdot, \cdot, f_{pe}(i,j)) \mid i \in \{0, \dots, T-1\},~ j \in \{0, \dots, T-1\} \,\}$, it would be impossible to identify $h(T)$ distinct functions whose cross-distances, measured by $dis_g$, all exceed $\epsilon$. Consequently, the model would be capable of distinguishing fewer than $h(T)$ pairwise positional signals, which directly contradicts the definition of $h(T)$. Therefore, the assumption in Eq.~\ref{eq:proof_1_1} does not hold. Eq.~\ref{eq:3_restate} in Proposition~\ref{prop:1_restate} can thus be proved.
\end{proof}

\section{More Details on the Brute-force Estimation of $I_q(n_{\text{head}}, n_{\text{scope}})$}

Inspired by \cite{zhang2025fast}, we observe that a brute-force estimation of $I_q(n_{\text{head}}, n_{\text{scope}})$ can be carried out as follows. For the $n_{\text{head}}$-th attention head, we keep all other heads at the full-attention scope and sequentially assign this head to each of its $N_{\text{scope}}$ candidate scopes. For every candidate scope, we perform one full generation pass of the DiT model and estimate $I_q(n_{\text{head}}, n_{\text{scope}})$ as the difference between the loss obtained under the full-attention configuration and the loss obtained under the $n_{\text{scope}}$-th candidate scope. Under this brute-force strategy, beyond the single generation pass where all heads use full attention, we additionally require $N_{\text{head}} \times N_{\text{scope}}$ further DiT model passes, rendering the estimation of $I_q(n_{\text{head}}, n_{\text{scope}})$ computationally prohibitive. Moreover, even if one of the $N_{\text{scope}}$ candidate scopes corresponds to the full-attention setting, we would still need $N_{\text{head}} \times (N_{\text{scope}} - 1)$ additional passes on top of the full-attention run. Thus, the computational burden remains prohibitive even in this case.

\section{Proof that the Slide Operation Yields a Unique Representation for Each Token Index}

In this section, we theoretically show that, the slide operation in SPA 
produces a unique positional representation for each token index \(i \in \{0,1,\dots,T-1\}\). Specifically, recall that given the bundle size $N$ for each middle bundle, the slide operation first constructs $N$ variants of the mapping function $\phi_{bundle}(\cdot)$, denoted as $\{\phi_{bundle}^{(N_1 = 1)}(\cdot), \dots, \phi_{bundle}^{(N_1 = N)}(\cdot)\}$, where:
\begin{equation}
\label{eq:bundle}
\setlength{\abovedisplayskip}{3pt}
\setlength{\belowdisplayskip}{3pt}
\phi_{bundle}(i) =
\begin{cases}
0, & 0 \le i < N_1, \\[3pt]
\displaystyle \left\lceil \dfrac{i + 1 - N_1}{N} \right\rceil, & N_1 \leq i < T,
\end{cases}
\end{equation}
For each token \(i\), the slide operation further constructs a structured representation of the token index $i$ as $\mathbf{v_{bundle}}(i)=\{\phi_{\text{bundle}}^{(N_1 = 1)}(i),\dots,\phi_{\text{bundle}}^{(N_1 = N)}(i)\}$. Below, we prove that $\mathbf{v_{bundle}}(i)$ is guaranteed to be unique for every $i$. To achieve this proof, we first introduce a lemma.

\begin{lemma}
\label{lemma:sum}
For any \(i \in \{0, 1, \dots, T-1\}\), the sum of the bundle indices in $\mathbf{v_{bundle}}(i)$ is:
\begin{equation}
\sum_{n=1}^N \phi_{\text{bundle}}^{(N_1=n)}(i) = i
\end{equation}
\end{lemma}

\begin{proof}
We consider two cases based on the value of \(i\) relative to \(N\). 

\textbf{Case 1: \(0 \le i < N\).} For a fixed \(i\), we split the sum over \(n = \{1, 2, \dots, N\}\) into two sub-ranges. \textit{Sub-range 1: \(n > i\)}: In this case, we have \(\phi_{\text{bundle}}^{(N_1 = n)}(i) = 0\). There are \(N - i\) such \(n\). \textit{Sub-range 2: \(n \le i\)}: Here, \(\phi_{\text{bundle}}^{(N_1 = n)}(i) = \left\lceil \frac{i + 1 - n}{N} \right\rceil\). Here, for $n \le i < N$, we have $n + 1 \le i + 1 \le N$ and thus have $1 \le i + 1 - n \le N$. This gives us $\left\lceil \dfrac{i + 1 - n}{N} \right\rceil = 1$. There are \(i\) such \(n\). The total sum in this case is then:
\begin{equation}
\sum_{n=1}^N \phi_{\text{bundle}}^{(N_1=n)}(i) = (N - i) \cdot 0 + i \cdot 1 = i.
\end{equation}

\textbf{Case 2: \(i \ge N\).} In this case, for all \(n = \{1, 2, \dots, N\}\), we have \(i \ge n\), so:
$\phi_{\text{bundle}}^{(N_1=n)}(i) = \left\lceil \frac{i + 1 - n}{N} \right\rceil$. The sum is then:
\begin{equation}
\sum_{n=1}^N \phi_{\text{bundle}}^{(N_1=n)}(i) = \sum_{n=1}^N \left(\left\lceil \frac{i + 1- n}{N} \right\rceil \right).
\end{equation}
Let \(l = i + 1 - n\), so as \(n\) goes from 1 to \(N\), \(l\) goes from \(i \) down to \(i + 1 - N\). The sum becomes: $\sum_{l=i+1-N}^{i} \left\lceil \frac{l}{N} \right\rceil$. Then let \(i = qN + r\), where \(q = \left\lfloor \frac{i}{N} \right\rfloor\) and \(r = i \mod N\) (\(0 \le r < N\)). Then the range of \(l\) is from \(i + 1 - N = (q-1)N + r + 1\) to \(i = qN + r\), which we also divide into two sub-ranges.
\textit{Sub-range 1: \(l = (q-1)N + r + 1\) to \(qN\)}: There are \(N - r\) terms. For these \(l\), \(\frac{l}{N} \le \frac{qN}{N} = q\), but since \(l > (q-1)N\), \(\left\lceil \frac{l}{N} \right\rceil = q\). \textit{Sub-range 2: \(l = qN + 1\) to \(qN + r\)}: There are \(r\) terms, and \(\left\lceil \frac{l}{N} \right\rceil = q + 1\). Thus:
\begin{equation}
\begin{aligned}
\sum_{n=1}^N \phi_{\text{bundle}}^{(N_1=n)}(i)
&= \sum_{n=1}^N \left\lceil \frac{i + 1 - n}{N} \right\rceil  \\
&= \sum_{l=i+1-N}^{i} \left\lceil \frac{l}{N} \right\rceil  \\
&= (N - r) \cdot q + r \cdot (q + 1)  \\
&= qN - qr + qr + r  \\
&= qN + r  \\
&= i .
\end{aligned}
\end{equation}

The above two cases together prove Lemma~\ref{lemma:sum}.
\end{proof}

We now leverage Lemma~\ref{lemma:sum} to prove uniqueness. Assume, for the sake of contradiction, that two distinct token indices \(i\neq j\) have $\mathbf{v_{bundle}}(i) = \mathbf{v_{bundle}}(j)$. Applying Lemma~\ref{lemma:sum} we have $i = \sum_{n=1}^N \phi_{\text{bundle}}^{(N_1=n)}(i) = \sum_{n=1}^N \phi_{\text{bundle}}^{(N_1=n)}(j) = j$, which contradicts $i\neq j$. Thus, every token index \(i\) is guaranteed to correspond to a distinct (unique) $\mathbf{v_{bundle}}(i)$.

\end{document}